\pdfoutput=1  %% arXiv: build with pdflatex
\documentclass[11pt]{article}

\usepackage{preprint}                 % local style: layout, title block, header
\usepackage[T1]{fontenc}
\usepackage[utf8]{inputenc}
\usepackage{amsmath,amssymb,amsthm}
\usepackage{amsfonts,dsfont,mathrsfs}
\usepackage{lmodern}                  % Computer Modern (LaTeX default), vector T1
\usepackage{graphicx}
\usepackage{subcaption}
\usepackage{booktabs}
\usepackage{array}
\usepackage{microtype}
\usepackage{caption}

\usepackage{algorithm}
\usepackage{algorithmic}

\usepackage[round]{natbib}            % \citep / \citet, author-year

\usepackage[hyphens]{url}
\usepackage[colorlinks=true,linkcolor=linkblue,citecolor=linkblue,
            urlcolor=linkblue]{hyperref}
\definecolor{linkblue}{RGB}{0,60,140}
\definecolor{thblue}{RGB}{65,102,245}

\usepackage{pgfplots}
\pgfplotsset{compat=newest}
\definecolor{grsone}{RGB}{31,119,180}
\definecolor{grstwo}{RGB}{255,127,14}
\definecolor{grsfour}{RGB}{44,160,44}
\definecolor{grseight}{RGB}{214,39,40}

\newtheoremstyle{preprintthm}%
  {\topsep}{\topsep}%
  {\itshape}{0pt}%
  {\normalsize\headingfont}{.}%
  {0.5em}%
  {\thmname{#1}\thmnumber{ #2}\thmnote{ {\normalfont(#3)}}}

\theoremstyle{preprintthm}
\newtheorem{theorem}{Theorem}
\newtheorem{remark}{Remark}

\theoremstyle{plain}          % everything below keeps the old look

\colorlet{thmrule}{thblue}
\colorlet{thmbg}{thblue!5!white}

\tcbset{thmbox/.style={
    enhanced, breakable,
    colback=thmbg,
    frame hidden, boxrule=0pt,
    borderline west={2pt}{0pt}{thmrule},
    arc=2pt,
    left=1em, right=1em, top=0.6em, bottom=0.6em, boxsep=0pt,
    before skip=\topsep, after skip=\topsep}}

\tcbset{rulebox/.style={thmbox,
    opacityback=0,
    left=0.9em, right=0pt,
    top=0.35em, bottom=0.35em}}

\tcolorboxenvironment{theorem}{thmbox}
\newcommand*\diff{\mathop{}\!\mathrm{d}}

\newcommand{\gmdim}{500}
\newcommand{\gmchurn}{0.1}

\newcommand{\rt}{\mathrm{r}}
\newcommand{\pa}{\mathrm{pa}}
\newcommand{\Verts}{\mathcal{V}}
\newcommand{\ch}{\mathcal{C}}
\newcommand{\Tree}{\mathcal{T}}
\newcommand{\Nacc}{N_\alpha}
\newcommand{\Verify}{\textsc{Verify}}

\usetikzlibrary{calc}
\definecolor{phaseonebase}{RGB}{222,161,147}    % blue
\definecolor{phasetwobase}{RGB}{47, 222, 161}    % purple
\definecolor{phasethreebase}{RGB}{108,59,170}  % lime
\colorlet{phasedraft}{phaseonebase!10!white}
\colorlet{phaseverify}{phasetwobase!10!white}
\colorlet{phaseaccept}{phasethreebase!12!white}
\colorlet{phasedraftink}{phaseonebase!90!black}
\colorlet{phaseverifyink}{phasetwobase!85!black}
\colorlet{phaseacceptink}{phasethreebase!85!black}
\newcommand{\algphase}[2]{\textcolor{#1}{\textit{// #2}}}
\newcommand{\phasename}[2]{\textcolor{#1}{\textbf{#2}}}
\newcommand{\phasemark}[1]{\tikz[remember picture,overlay,baseline]{\coordinate (#1);}}
\newcommand{\phaseband}[3]{%  #1 = colour, #2 = first mark, #3 = mark on the line after
  \begin{tikzpicture}[remember picture,overlay]
    \fill[#1,rounded corners=2pt,blend mode=multiply]
      ([shift={(-2.3em,1.6ex)}]#2)
      rectangle
      ([shift={(\dimexpr\linewidth-2.3em\relax,2.0ex)}]#3);
  \end{tikzpicture}}

\title{Accelerating Diffusion Sampling via\\ Speculative Draft Trees}
\runningtitle{Accelerating Diffusion Sampling via Speculative Draft Trees}
\newcommand{\ICLline}{Imperial College London}

\newcommand{\ICLfull}[1]{\scriptsize\ICLline\\ \scriptsize\texttt{#1@imperial.ac.uk}}

\newcommand{\authorsGrid}{%
  \small
  \renewcommand{\authorcolsep}{0.05em}%
  \renewcommand{\arraystretch}{0.82}%
  \begin{tabular}{c@{\hspace{2em}}c@{\hspace{2em}}c}
    \AuthorCol{Marcello Bullo}{\ICLfull{m.bullo21}} &
    \AuthorCol{Yanxiao Liu}{\ICLfull{y.liu2}} &
    \AuthorCol{\"Oyk\"u S\i la G\"uner}{\ICLfull{o.guner25}}
  \end{tabular}\\[0.9em]
  \begin{tabular}{c@{\hspace{2em}}c}
    \AuthorCol{Arpan Mukherjee}{\ICLfull{a.mukherjee}} &
    \AuthorCol{Deniz G\"und\"uz}{\ICLfull{d.gunduz}}
  \end{tabular}}

\author{\authorsGrid}

\begin{document}
\maketitle
\allowdisplaybreaks

\begin{abstract}
    Speculative sampling accelerates diffusion model generation by drafting inexpensive candidate states and correcting them under a coupling that preserves the target distribution {\em exactly}, reducing the number of expensive target evaluations. Existing diffusion samplers, notably those based on reflection maximal coupling, are topologically constrained: their lookahead drafts form a chain graph, a single linear sequence, which inherently limits the acceptance rate per target evaluation. We connect speculative sampling in diffusion models to relative entropy coding (REC). This perspective shows the lookahead need not be linear and motivates our central contribution, {\em draft trees}, which enrich the candidates considered per round and lower the target function evaluations. We further adopt greedy rejection sampling, an REC algorithm, as the draft–target coupling, improving acceptance while guaranteeing exact target samples. Experiments across diverse target and draft models demonstrate up to 8.3\% acceleration over the reflection coupling baseline in practical settings.
\end{abstract}

%% TODO: replace with the real repository URL before posting.
\codelink{https://github.com/marcellobullo/tree-specdiff}


\section{Introduction}
\label{sec::intro}
Denoising diffusion probabilistic models (DDPMs)~\citep{sohl2015deep,ho2020denoising,song2021scorebased} have become the dominant paradigm for high-fidelity image generation. By learning to reverse a gradual noising process, diffusion models synthesize images through a sequence of denoising steps that progressively transform Gaussian noise into a sample from the target data distribution. Their remarkable generation quality has established diffusion models as the state of the art across a wide range of generative tasks \citep{yang2023diffusion}. 

\begin{figure}[t]
    \centering
    \includegraphics[width=.7\linewidth]{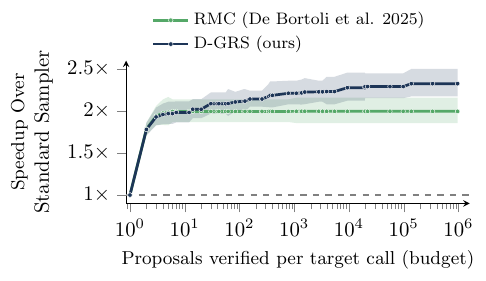}
    \caption{\textbf{Efficiency frontier of speculative sampling.} The plot compares the maximal speedup attainable over a standard sampler at a given proposal budget $B$ for the setting described in Section \ref{sec:exp-gm}. By allocating budget strictly to sequential depth, RMC rapidly saturates near a $2.00\times$ speedup. In contrast, D-GRS allocates budget to multiple proposals per node, effectively raising the per-step acceptance probability. This branching strategy allows D-GRS to monotonically convert larger budgets into measurable acceleration, up to a $14\%$ reduction in target model evaluations relative to RMC.}
    \label{fig:intro-speedup}
\end{figure}
% Figure~\ref{fig:intro-speedup} shows the resulting efficiency frontier, i.e.\ the best speedup attainable at a proposal budget of at most $B$. RMC improves quickly up to $B\approx10$ and then \emph{saturates}: a longer draft chain does not help, because once a proposal is rejected the remainder of the chain is discarded, and the probability that a chain survives $\ell$ steps decays geometrically[CITE]. 

% Its best cell over the entire grid is $15.53\pm1.17$ calls ($2.00\times$ speedup, at $B=1092$), statistically indistinguishable from what it already achieves at $B=14$ ($15.73$ calls). 

% D-GRS converts the same budget into width rather than depth and keeps improving monotonically with $B$: $15.34$ calls at $B=14$, $14.25$ at $B=340$, $13.91$ at $B=2800$ and $13.33\pm0.93$ at its best cell
% ($K=7$, $L=6$), 

% a $2.33\times$ speedup over the standard sampler and a
% $14\%$ reduction in NFE relative to the best RMC configuration. 

% The gap widens with the budget, which is the behaviour the tree construction is designed to produce: additional children raise the per-step acceptance probability, whereas additional chain depth only extends a walk that has already committed.
A major drawback of diffusion models, however, is their high inference cost. Image generation requires repeatedly evaluating a large neural network over tens to hundreds of denoising iterations, making sampling extremely slow. This computational bottleneck has motivated a broad range of acceleration techniques. Training-free approaches include improved solvers to reduce the number of sampling steps~\citep{zhang2022fast,lu2025dpm}, while training-based methods learn few-step generators, either by preserving the sampling trajectory, as in consistency models~\citep{song2023consistency} and their two-time flow-map
generalizations~\citep{geng2026improved,zhou2026terminal,sabour2026align}, recently extended to learn the solution map of the underlying SDE~\citep{mccallum2026strong}, or by directly matching the output distribution~\citep{yin2024improved,sauer2024adversarial}. While effective in reducing latency, these approaches typically incur either additional training cost or degradation in sample quality~\citep{dieleman2026flowmaps,lai2025principles}. 
Parallel simulation provides an orthogonal acceleration strategy, trading additional computation and memory for reduced latency~\citep{shih2023parallel,tang2024accelerating,chen2024accelerating,zhou2025parallel}. Despite reducing the sequential depth of diffusion sampling, it typically requires multiple iterative refinement rounds and substantial parallel resources.
% Parallel simulation methods~\citep{shih2023parallel,chen2024accelerating} provide an orthogonal direction by exploiting parallel computation, yet remain fundamentally iterative, and therefore offer limited acceleration.

Recently, \emph{speculative diffusion sampling}~\citep{de2025accelerated} has emerged as a complementary paradigm for accelerating diffusion inference. Inspired by speculative decoding for large language models~\citep{leviathan2023fast,chen2023accelerating}, these methods first generate candidate denoising states using an inexpensive source (referred to as a {\em proposal}) before verifying them using the target diffusion model. Whenever multiple draft states are accepted, several denoising iterations are effectively skipped, leading directly to lower inference latency. Crucially, unlike approximate acceleration techniques, speculative diffusion sampling is \emph{exact}: it preserves the sampling distribution of the target model, and therefore incurs no loss in image quality.

Despite these advantages, existing speculative diffusion samplers are fundamentally \emph{topology constrained}. They generate drafts sequentially along the diffusion trajectory, resulting in a linear chain of proposal states. From a broader perspective, speculative sampling is a \emph{budget allocation} problem: given a fixed amount of parallel computation, how should proposal evaluations be allocated to maximize the probability of accepting draft states? Restricting proposals to a chain represents only one possible allocation strategy and can leave substantial parallelism underutilized, particularly in the low-acceptance regime that commonly arises in practice.

In this paper, we take a first step toward lifting this restriction. Our contributions are two-fold. First, we generalize speculative diffusion sampling from linear draft chains to \emph{draft trees}, enabling substantially richer allocations of parallel computation across candidate denoising trajectories. Second, we establish a novel connection between speculative sampling and \emph{relative entropy coding} (REC)~\citep{li2024channel,li2018strong,flamich2023adaptive,flamich2024greedy}. This connection leads naturally to a greedy rejection sampling (GRS)~\citep{harsha2010communication} strategy for budget allocation, providing an efficient and principled mechanism for constructing draft trees. Extensive experiments demonstrate that our approach consistently improves speculative diffusion sampling, achieving speedups of up to 7.4\% over existing methods for practical batch sizes $\leq 256$, while preserving exact sampling (see Figure~\ref{fig:intro-speedup}).

\section{Problem Setting}
\paragraph{DDPM.}
DDPMs define a fixed forward diffusion process that gradually perturbs data samples with Gaussian noise. Let $X_0\sim p_{\rm data}$ denote a data sample. The forward process $(X_t)_{t\in[0,1]}$ evolves according to the stochastic differential equation (SDE)
\begin{align}
\diff X_t
=
f_t(X_t)\,\diff t
+
g_t\,\diff B_t\ ,
\end{align}
where $(B_t)_{t\in[0,1]}$ is a $d$-dimensional Brownian motion. Image generation is performed by simulating the associated reverse-time process $(Y_t)_{t\in[0,1]}$, which satisfies
\begin{align}
%\begin{aligned}
\label{eq:sde}
\diff Y_t
&=
b_t^{q}(Y_t)\,\diff t
+
\varepsilon g_{1-t}\,\diff W_t\,
\end{align}
\begin{align}\label{eq:reverse_drift}
b_t^{q}(x)
&=
-f_{1-t}(x)
+
\frac{1+\varepsilon^2}{2}
g_{1-t}^2
s_{1-t}(x)\ ,
%\end{aligned}
\end{align}
where $(W_t)_{t\in[0,1]}$ is a $d$-dimensional Brownian motion,
$s_t(x)$ is the Stein score function, and
$\varepsilon\in[0,1]$ is the \emph{churn parameter} controlling the
stochasticity of the reverse process. In particular, $\varepsilon=1$
recovers the standard reverse-time SDE, whereas $\varepsilon=0$ yields the
deterministic probability-flow ordinary differential equation (ODE). To simulate this reverse process, we introduce a uniform time grid
\[
t_n=n\gamma\ ,
\qquad
n=0,\ldots,N\ ,
\qquad
\gamma=\frac{1}{N}\ ,
\]
and write $Y_n:=Y_{t_n}$. Applying the Euler--Maruyama discretization with $\xi_n\sim\mathcal N(0,\mathbb I_d)$ gives
\begin{align}\label{eq:em_discretisation}
Y_{n+1}
=
Y_n
+
\gamma b_{t_n}^{q}(Y_n)
+
\sqrt{\gamma}\,\varepsilon g_{1-t_n}\,\xi_n\ .
\end{align}
Hence, conditional on $Y_n=y_n$, the next state follows $Y_{n+1}\,\big|\, Y_n=y_n \sim \mathcal N\!\left(m_n^{q}(y_n), \sigma_n^2\mathbb I_d\right)$, where
\begin{align}\label{eq::target_mean}
m_n^{q}(y)
:=
y+\gamma b_{t_n}^{q}(y) \quad\text{and}\quad
\sigma_n
:=
\sqrt{\gamma}\,\varepsilon g_{1-t_n}\ .
\end{align}
Equivalently, denoting the target transition kernel by
$\mathbb Q_n(\cdot\mid y_n)$, we have $\mathbb Q_n(\cdot\mid y_n) = \mathcal N\!\left(m_n^{q}(y_n), \sigma_n^2\mathbb I_d \right),$ and the discretized reverse process forms a Markov chain with joint density
\begin{align}
q(y_0,\ldots,y_N)
=
q_0(y_0)
\displaystyle\prod_{n=0}^{N-1}
q_n(y_{n+1}\mid y_n)\ ,
\end{align}
where $q_0$ denotes the initial Gaussian density of the reverse process.

\paragraph{Speculative diffusion sampling.}
The principal computational bottleneck in diffusion inference is the repeated evaluation of the reverse drift $b_{t_n}^{q}$, which is parameterized by a large neural network. Since Euler--Maruyama simulation proceeds sequentially, each state $Y_{n+1}$ depends on the reverse drift evaluated at the preceding state $Y_n$, limiting the utilization of modern parallel hardware. Speculative diffusion sampling addresses this bottleneck by exploiting parallel computation. Rather than simulating one denoising step at a time, an inexpensive \emph{proposal} model first generates multiple draft states. The key assumption underlying speculative sampling is that \emph{drafting is significantly cheaper than verification}. The drafted states are subsequently verified using the target reverse drift in parallel. Consequently, given a parallel compute budget of $B$, one batched target-model evaluation, corresponding to one neural function evaluation (NFE), can verify all proposed drafts.

The effectiveness of speculative sampling depends critically on the quality of the proposal model. Better proposals lead to higher acceptance rates, allowing more denoising steps to be advanced per NFE and thereby increasing the overall speedup while preserving exact sampling. Existing speculative diffusion methods~\citep{de2025accelerated} consider two proposal mechanisms. The first trains a lightweight diffusion model to approximate the target reverse drift. The second, referred to as the \emph{delayed reverse drift}, avoids additional training by replacing the current reverse drift with one evaluated at an earlier state. Specifically, the target transition mean $m_n^{q}(y_n)=y_n+\gamma \textcolor{red}{b_{t_n}^{q}(y_n)}$ is approximated by the proposal mean
\begin{align}\label{eq:delayed_drift}
m_n^{p}(y_n)
=
y_n+\gamma \textcolor{blue}{b_{t_k}^{q}(y_k)},
\qquad
k<n\ ,
\end{align}
while retaining the same transition variance $\sigma_n^2=\gamma\varepsilon^2g_{1-t_n}^2$. Empirically, delayed reverse drift has been shown to produce higher-quality proposals than lightweight draft models~\citep{de2025accelerated}. Accordingly, we adopt it as the proposal mechanism throughout this paper.

    \section{Speculative Algorithms}
\label{sec:speculative_algorithms}
In this section, we first state the general template that both algorithms instantiate, then review the reflection maximal coupling (RMC) algorithm of~\citep{de2025accelerated}, and finally introduce our diffusion-GRS (D-GRS) algorithm. At a high level, every speculative sampling algorithm consists of three phases: a \emph{drafting} phase, a \emph{verification} phase, and an \emph{acceptance} phase. Given a parallel compute budget, the drafting phase sequentially generates candidate states according to a prescribed draft topology, the verification phase evaluates all drafted states in parallel using the target model, and the acceptance phase sequentially examines the verified states to determine whether a proposal can be accepted as the next target state. Consequently, speculative sampling algorithms differ primarily in the choice of \textbf{draft topology} and the \textbf{acceptance criterion} used after verification. Both RMC and D-GRS rely on a common one-dimensional reduction that exploits the translation and rotational invariance of isotropic Gaussian distributions, described next. Throughout, we use $\widehat Y$ (or $\widehat S$) to denote a proposed state and $Y$ (or $S$) to denote the accepted target state.

\subsection{A General Template For Speculative Diffusion}
We begin from the general procedure, stated so that the two algorithms of this section differ only in the two components the paper varies: the \emph{draft
topology} and the \emph{verification rule}. The template is agnostic to how proposals are produced and it is shown in Algorithm~\ref{alg:template}.

\begin{algorithm}[!t]
\caption{Speculative diffusion sampling for an arbitrary draft tree $\Tree$. 
%In Phase~1 the inner loop ranges over $\Verts_{\ell-1}$, the parents that spawn layer $\ell$; for $\ell=1$ this is the root alone.
}
\label{alg:template}
\begin{algorithmic}[1]
\REQUIRE draft and target mean functions $m^p$, $m^q$, scales $\{\sigma_n\}$, horizon
$N$, draft tree $\Tree=(\Verts,\pa)$ of depth $L$, rule $\Verify$
\ENSURE trajectory $Y_{1:N}$ with the law~(6) of the target chain
\STATE $n \leftarrow 0$; \ $Y_0 \sim q_0$
\WHILE{$n < N$}
  \STATE $L_n \leftarrow \min\{L, N-n\}$, $\widehat Y_\rt \leftarrow Y_n$, $\Tree_n \leftarrow \Tree|_{L_n}$, $\mathcal{I}_n \leftarrow \mathcal{I}(\Tree_n)$
  \STATE \phasemark{phdraft}\algphase{phasedraftink}{Phase 1 -- Drafting}
  \FOR{$\ell = 1$ to $L_n$}
    \FORALL{$u \in \Verts_{\ell-1}$ \textbf{in parallel}}
      % \STATE draw $\widehat Y_{v} \sim \mathcal N\big(m^p_{n+\ell}(\widehat Y_u),\,
      %        \sigma^2_{n+\ell} I_d\big)$ independently for each $v \in \ch(u)$
       \STATE $\big(\widehat Y_{v}\big)_{v \in \ch(u)} \overset{\text{i.i.d.}}{\sim} \mathcal N\big(m^p_{n+\ell-1}(\widehat Y_u),\,
             \sigma^2_{n+\ell-1} I_d\big)$
    \ENDFOR
  \ENDFOR
  \STATE \phasemark{phverify}\algphase{phaseverifyink}{Phase 2 -- Verification}
  \FORALL{$u \in \mathcal I_n$ \textbf{in parallel}}
    \STATE evaluate $m^q_{n+|u|}\big(\widehat Y_u\big)$
  \ENDFOR
  \STATE \phasemark{phaccept}\algphase{phaseacceptink}{Phase 3 -- Acceptance}
  \STATE $u \leftarrow \rt$; \ $\Nacc \leftarrow 0$
  \FOR{$\ell = 1$ to $L_n$}
    \STATE $\big(Y_{n+\ell},\,\texttt{accepted},\,v^\star\big) \leftarrow
           \Verify\big(m^p_{n+\ell-1}(\widehat Y_u),\, m^q_{n+\ell-1}(\widehat Y_u),\,
           \sigma_{n+\ell-1};\, \{\widehat Y_{v}\}_{v \in \ch(u)}\big)$
    \IF{\NOT \texttt{accepted}}
     \STATE $\Nacc \leftarrow \Nacc + 1$
      \STATE \textbf{break} %   \COMMENT{residual draw: no children}
    \ENDIF
    \STATE $\Nacc \leftarrow \Nacc + 1$; \ $u \leftarrow v^\star$
  \ENDFOR
  % \STATE $n \leftarrow n + \min\{\Nacc + 1,\, L_n\}$
  \STATE \phasemark{phend}$n \leftarrow n + \min\{\Nacc,\, L_n\}$
\ENDWHILE
\RETURN $Y_{1:N}$
\end{algorithmic}
\phaseband{phasedraft}{phdraft}{phverify}%
\phaseband{phaseverify}{phverify}{phaccept}%
\phaseband{phaseaccept}{phaccept}{phend}%
\end{algorithm}

\subsubsection{Setting}
Fix a round beginning at step $n$. The template requires only that the draft and target transitions be isotropic Gaussians sharing the variance schedule and
differing in their means,
\begin{equation}
  \begin{aligned}
    \mathbb P_n(\cdot\mid y) &= \mathcal N\big(m^p_n(y),\sigma_n^2 I_d\big), \\
    \mathbb Q_n(\cdot\mid y) &= \mathcal N\big(m^q_n(y),\sigma_n^2 I_d\big),
  \end{aligned}
  \label{eq::kernels}
\end{equation}
with $m^q_n$ as in~\eqref{eq::target_mean} and $m^p_n$ arbitrary. This is exactly the structure the
rank-$1$ reduction in \ref{par:rank1} exploits, so both Algorithm~\ref{alg:rmc} and Algorithm~\ref{alg:grs} below apply
unchanged at every node.

\subsubsection{Draft Topology}
A \emph{node} is one candidate realisation of the process state at one step. Its depth is
the step it realises and its parent is the state it was conditioned on. Because the reverse chain is Markov, each node is drawn conditionally on exactly one parent, so a draft set always carries the structure of a rooted tree, and we take
such a tree as the general description of a draft topology.

A \emph{draft tree} is a finite rooted tree $\Tree=(\Verts,\pa)$ with vertex set
$\Verts$, root $\rt$, and parent map $\pa:\Verts\setminus\{\rt\}\to\Verts$. We
write $\ch(u):=\pa^{-1}(u)$ for the children of $u$, $|u|$ for its depth, and
\begin{equation}\label{eq::nodes}
    \begin{aligned}
        &\Verts_0=\{\rt\},\\
        &\Verts_\ell:=\{u\in\Verts:\ |u|=\ell\},\\
        &\mathcal I(\Tree):=\{u\in\Verts:\ \ch(u)\neq\emptyset\},
    \end{aligned}
\end{equation}
for the layers and the internal nodes. The root holds the last accepted state
$\widehat Y_\rt=Y_n$ and node $u\neq\rt$ holds the drafted state $\widehat Y_u$, a candidate realisation of step $n+|u|$. The number of drafted states is $|\Verts|-1$, and the target model can be evaluated at all the tree nodes $u\in\Verts$, or only at the internal ones $u\in\mathcal I(\Tree)$. 
%The latter is sufficient for any acceptance criterion to take an informed decision, although the former 
%, while the target model is evaluated only at the internal nodes $\mathcal I$, since a leaf is never a parent. 
% We therefore take the \emph{budget} to be that batch, $B:=|\mathcal I|$.

The children of a node are i.i.d. candidates for the \emph{same} step, of which
verification retains at most one. Because the accepted candidate node is determined exclusively during the verification phase, the proposal mechanism must exhaustively expand all sibling branches. The uniform $(K,L)$ family used in the paper is the case
$|\ch(u)|=K$ for every $u$ of depth at most $L-1$, and $\Verts_\ell=\emptyset$ for $\ell>L$. Therefore, $|\Verts_\ell|=K^{\ell}$ and
\begin{equation}
  |\mathcal I(\Tree)|=\sum_{\ell=0}^{L-1}K^{\ell},
  \qquad
  |\Verts|=\sum_{\ell=0}^{L}|\Verts_\ell|=\sum_{\ell=0}^{L}K^{\ell}.
  \label{eq::uniform}
\end{equation}
During the verification phase, when the target model is evaluated at $u\in\Verts$, a round costs $B=|\Verts|$ target evaluations, while it costs $B=|\mathcal I(\Tree)|$ when the target model is evaluated at $u\in\mathcal I(\Tree)$.
The chain of RMC corresponds to the special case where $K=1$, while a single-node proposal list to $L=1$. Depth-dependent widths $K_\ell$, or irregular trees (e.g., pruned), are covered by $\Tree$ directly and require no change to what follows.

Let $\Verts_{\le m}:=\bigcup_{\ell=0}^{m}\Verts_\ell$
denote the set of nodes of depth at most $m$, and let
\begin{equation}
  \Tree|_{m}
  :=
  \Tree\big[\Verts_{\le m}\big]
  =
  \big(\Verts_{\le m},\ \pa|_{\Verts_{\le m}\setminus\{\rt\}}\big)
  \label{eq::truncation}
\end{equation}
denote the subtree of $\Tree$ induced by $\Verts_{\le m}$. 

A round starting at step $n$ therefore operates on $\Tree|_{L_n}$ with $L_n:=\min\{L,N-n\}$, which equals $\Tree$ in every round but the last few.

\subsubsection{Round Structure}
Algorithm~\ref{alg:template} describes the speculative diffusion sampling procedure for an arbitrary draft tree $\Tree$. Each iteration of the algorithm involves three distinct phases.
\begin{enumerate}
    \item The \phasename{phasedraftink}{drafting phase} executes a layer-wise expansion of the tree: this expansion is strictly sequential with respect to depth, as child nodes cannot be sampled prior to their respective parents. However, all nodes residing at the same depth level are generated concurrently.
    \item The \phasename{phaseverifyink}{verification phase} computes the target means for all internal nodes of the current tree via a unified batched operation, representing the sole evaluation of the target model within a given iteration. 
    \item Finally, the \phasename{phaseacceptink}{acceptance phase} traverses downward from the root node, iteratively applying the verification criterion at each depth level and terminating immediately upon encountering the first rejection.
\end{enumerate}
The verification rule, denoted as $\Verify$, serves as a modular component within this framework and is governed by the following formal contract. Given a parent node $u$, its associated proposal and target means, the corresponding step scale, and its set of drafted children $\{\widehat Y_v\}_{v\in\ch(u)}$, the rule outputs a tuple $(Y, \texttt{accepted}, v^\star)$. The primary component, $Y$, constitutes an exact sample drawn from the target transition distribution, $Y \sim \mathbb{Q}\big(\cdot \mid \widehat{Y}_u\big)$. This sample is obtained through one of two mutually exclusive mechanisms: either a valid child node is identified, wherein the boolean flag $\texttt{accepted}$ evaluates to true and $Y = \widehat{Y}_{v^\star}$ for a specific $v^\star \in \ch(u)$; alternatively, if no child satisfies the criterion, $\texttt{accepted}$ evaluates to false, and $Y$ is instead sampled from the normalized residual distribution formulated in~\eqref{eq:GRS_residual} below. Algorithm~\ref{alg:grs} fulfills this contract for any arbitrary branching factor $K \ge 1$, whereas Algorithm~\ref{alg:rmc} is specialized for $K=1$. 
%Both algorithms strictly necessitate the supplementary return of the accepted child, a value inherently accessible at their point of termination.

The template leaves exactly two components free: the draft topology $\Tree$ and the verification rule $\Verify$. The remainder of this section instantiates them. RMC takes $\Tree$ to be a chain ($K=1$) and $\Verify$ to be the reflection maximal coupling by \citep{de2025accelerated}; D-GRS takes $\Tree$ to be a $K$-ary tree and $\Verify$ to be our proposed D-GRS algorithm. Both rely on a common one dimensional reduction that exploits the translation and rotational invariance of isotropic Gaussian distributions, described next.

\subsubsection{Rank-1 Projection}\label{par:rank1}
Since the proposal $\mathbb{P}:=\mathcal{N}(\mu_p,\sigma^2\mathbb{I}_d)$ and target $\mathbb{Q}:=\mathcal{N}(\mu_q,\sigma^2\mathbb{I}_d)$ are isotropic Gaussians differing only in their means, we can reduce the $d$-dimensional coupling to a single dimension. Let
\begin{align}\label{eq:r1proj-variables}
    \Delta:=(\mu_q-\mu_p)/\sigma, \quad \delta:=\|\Delta\|, \quad e:=\Delta/\delta ,
\end{align}
be the normalized mean displacement, its Euclidean norm, and the corresponding unit vector, respectively.
Standardizing a state $Y$ as $Z:=(Y-\mu_p)/\sigma$ yields $Z\sim\mathcal{N}(0,\mathbb{I}_d)$ under $\mathbb{P}$ and $Z\sim\mathcal{N}(\delta e,\mathbb{I}_d)$ under $\mathbb{Q}$. Note that any realization $z\in\mathbb{R}^d$ admits the orthogonal decomposition
\begin{equation}
    z = e\underbrace{e^\top z}_{s} + \underbrace{(\mathbb{I}_d-ee^\top)z}_{z_\perp}\ .
\end{equation}
Under either distribution, the random variables $S:=e^\top Z$ and $Z_\perp$ are independent, and the orthogonal residual is $Z_\perp\sim\mathcal{N}(0,\mathbb{I}_d-ee^\top)$. {\em Only the scalar projection differs}: $S\sim\mathcal{N}(0,1)$ under $\mathbb{P}$, and $S\sim\mathcal{N}(\delta,1)$ under $\mathbb{Q}$. Consequently, given a proposal sample $\widehat Y\sim\mathbb{P}$, exact sampling from $\mathbb{Q}$ requires transforming only the scalar projection $\widehat{S}$ from $\mathcal{N}(0,1)$ to $\mathcal{N}(\delta,1)$, while retaining the proposal's orthogonal component $Z_\perp$ unchanged:
\begin{equation}\label{eq:rank1-projector}
    \widehat{S} := e^\top(\widehat{Y}-\mu_p)/\sigma, \quad Z_\perp = (\mathbb{I}_d-ee^\top)(\widehat{Y}-\mu_p)/\sigma\ .
\end{equation}
Once a valid target scalar $S\sim\mathcal{N}(\delta,1)$ is obtained, the final state is reconstructed as
\begin{equation}\label{eq:final_proj}
    Y = \mu_p + \sigma(S e+Z_\perp)\ .
\end{equation}
One of the major distinctions between RMC and GRS lies in how they transform $\widehat S\sim\mathcal{N}(0,1)$ into $S\sim\mathcal{N}(\delta,1)$.

% \begin{algorithm}[H]
% \caption{\texttt{Rank-1 Projector}}
% \label{alg:rank-one-projector}
% \begin{algorithmic}[1]
% \small
% \REQUIRE proposal sample $\widehat Y\sim\mathbb{P}=\mathcal{N}(\mu_p,\sigma^2\mathbb{I}_d)$; target mean $\mu_q$; shared scale $\sigma$
% \ENSURE scalar projection $\widehat S$; orthogonal component $Z_\perp$; displacement magnitude $\delta$; displacement direction $e$

% \STATE $\Delta \gets (\mu_q-\mu_p)/\sigma$
% \STATE $\delta \gets \lVert\Delta\rVert$
% \STATE $e \gets \Delta/\delta$

% \STATE $Z \gets (\widehat Y-\mu_p)/\sigma$
% \STATE $\widehat S \gets e^\top Z$
% \COMMENT{projection parallel to $e$}
% \STATE $Z_\perp \gets (\mathbb{I}_d-ee^\top)Z$
% \COMMENT{projection orthogonal to $e$}

% \RETURN $\widehat S,Z_\perp,\delta,e$
% \end{algorithmic}
% \end{algorithm}

\subsection{Reflection Maximal Coupling (RMC)}
RMC~\citep{de2025accelerated} instantiates the speculative sampling framework using a {\em chain topology} and {\em reflection maximal coupling} for acceptance. Specifically, the drafting phase sequentially generates a chain of $L$ proposal states, where $L$ is referred to as the \emph{lookahead}. During the acceptance phase, each proposal is maximally coupled with the corresponding target transition, thereby minimizing the probability of rejecting the drafted state.

\begin{algorithm}[t]
\caption{Reflection maximal coupling (RMC)}
\label{alg:rmc}
\begin{algorithmic}[1]
\small
\REQUIRE proposal sample
$\widehat Y\sim\mathbb{P}
=\mathcal{N}(\mu_p,\sigma^2\mathbb{I}_d)$;
target mean $\mu_q$;
shared scale $\sigma$
\ENSURE target sample
$Y\sim\mathbb{Q}
=\mathcal{N}(\mu_q,\sigma^2\mathbb{I}_d)$;
acceptance indicator $\texttt{accepted}\in\{\texttt{true},\texttt{false}\}$

\STATE Obtain $(\widehat S,Z_\perp,\delta,e)$ from \eqref{eq:r1proj-variables} and \eqref{eq:rank1-projector}
\COMMENT{rank-1 projection}
% \gets
% \texttt{Rank-1 Projector}
% (\widehat Y,\mu_p,\mu_q,\sigma)$

\STATE $R\gets
\dfrac{\phi(\widehat S-\delta)}
{\phi(\widehat S)}$
\COMMENT{target-to-proposal likelihood ratio}

\STATE $U\sim\mathrm{Uniform}[0,1]$

\STATE $\texttt{accepted}
\gets
\mathds{1}\{U\leq 1\wedge R\}$

\IF{not $\texttt{accepted}$}
%     \STATE $S\gets\widehat S$
%     \COMMENT{accept proposal}
% \ELSE
    % \STATE $S\gets\delta-\widehat S$
    \STATE $\widehat S\gets\delta-\widehat S$
    \COMMENT{reflect $\widehat S$ about $\delta/2$}
\ENDIF

\STATE $Y\gets\mu_p+\sigma(\widehat S e+Z_\perp)$
\COMMENT{$Y\sim\mathbb Q$}

\RETURN $Y, \;\texttt{accepted}, \; 1$
\end{algorithmic}
\end{algorithm}

Given a proposal $\widehat Y$, RMC first decomposes it into a scalar proposal $\widehat S$ along the displacement direction and an orthogonal component $Z_\perp$ using~\eqref{eq:r1proj-variables} and~\eqref{eq:rank1-projector}. The scalar proposal is accepted with probability equal to the target-to-proposal likelihood ratio, capped at one. If accepted, RMC retains $\widehat S$ and proceeds to the next proposal in the chain. Otherwise, it draws an exact sample from the residual distribution and terminates the current speculative trajectory, after which a new lookahead chain is drafted. For Gaussian proposal and target distributions with equal covariance, residual sampling admits a simple geometric implementation: the rejected scalar proposal is reflected about the midpoint between the proposal and target means. The resulting scalar is then combined with the unchanged orthogonal component to reconstruct an exact sample from the target distribution. The complete procedure is summarized in Algorithm~\ref{alg:rmc}. Here, $\phi$ denotes the standard Gaussian density function. In line~$8$, the residual sample is obtained by reflecting $\widehat S$ about the midpoint $\delta/2$, yielding $S=\delta-\widehat S$.

\textit{Drawbacks of RMC:} Despite its simplicity and exactness, RMC is fundamentally constrained by its chain topology. Given a fixed parallel compute budget of $B$ draft evaluations, RMC allocates the entire budget to a single lookahead trajectory of length $B$ (setting $L=B$). Consequently, if a proposal is rejected early in the acceptance phase, all subsequent drafted states become unusable and are discarded, resulting in underutilization of the available parallel compute budget. This inefficiency becomes increasingly pronounced as the compute budget grows. To better exploit parallelism, we instead propose a tree-based draft topology that allocates the same budget across multiple speculative trajectories, as described next.

\subsection{Diffusion Greedy Rejection Sampling (D-GRS)}
GRS departs fundamentally from RMC in both its drafting topology and, consequently, its acceptance criterion.

\begin{figure}[t]
    \centering
    \includegraphics[width=.9\linewidth]{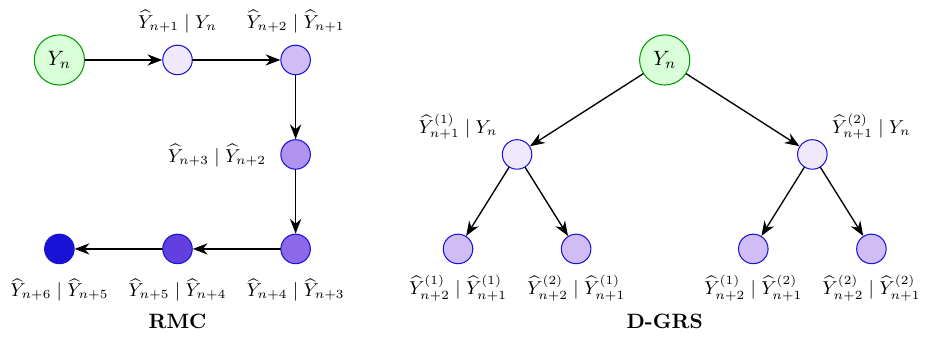}
    \caption{Allocation of a fixed compute budget under chain- and tree-structured drafting for $B=|\Verts|$. Given a compute budget of $W=7$ states, RMC constructs a single proposal chain. If an early proposal is rejected, all downstream proposals become unusable and are discarded. In contrast, D-GRS constructs a $K$-ary draft tree, here with $K=2$ and $L=2$. At each level, D-GRS may select among multiple candidate states, increasing the probability that at least one proposal can be accepted.}
    \label{fig:proposal_tree}
\end{figure}

\textit{Topology:}
The chain employed by RMC represents only one possible allocation of a fixed parallel compute budget. To exploit this budget more effectively, we introduce a more flexible topology based on \emph{draft trees}, depicted in Figure~\ref{fig:proposal_tree}. A $K$-ary draft tree expands every non-leaf node into $K$ children, each representing a candidate realization of the next target state. Suppose drafting begins at time step $n$, and let $Y_n$ denote the current accepted target state, obtained either from the preceding draft tree or through residual sampling. We treat $Y_n$ as the root of the tree and draw $K$ conditionally independent and identically distributed (i.i.d.)\ proposals from the proposal transition kernel at time $n+1$. Each resulting proposal is then recursively expanded into $K$ children, yielding a tree with $L$ levels, excluding the root. Defining $W$ as the number
of target evaluations that fit in a single parallel batch, the pair $(K,L)$ must satisfy
\begin{align}\label{eq:budget}
    B\leq W,
\end{align}
with $B=|\Verts|$ or $B=|\mathcal I(\Tree)|$, based on the verification mode.
% Thus, given a budget $B$ on the total number of states in the tree, the branching factor $K$ and number of levels $L$ are chosen such that
% \begin{align}\label{eq:budget}
%     1+K+\cdots+K^L\leq B\ .
% \end{align}
% \begin{align}\label{eq:budget}
%     B = K+\cdots+K^L=\frac{K^{L+1}-K}{K-1}.
% \end{align}
% Equivalently, excluding the already available root state, the tree contains
% \[
% K+\cdots+K^{L-1}
% =
% \frac{K^L-K}{K-1}
% \]
% drafted states.

During drafting, every node at a given level generates $K$ children using a proposal transition constructed from the target reverse drift already evaluated at the tree root $Y_n$. Since the target model has not yet been evaluated at the intermediate parent states, the same root drift is reused throughout the draft tree. Conditional on a parent state, its $K$ children are sampled independently from this delayed-drift proposal kernel. Once the full tree has been constructed, all drafted states are passed to the target model in parallel, yielding the parent-specific reverse drifts required for verification.

\begin{algorithm}[t]
\caption{Diffusion Greedy Rejection Sampling (D-GRS)}
\label{alg:grs}
\begin{algorithmic}[1]
\small
\REQUIRE proposals
$\widehat Y_{1:K}\overset{\mathrm{i.i.d.}}{\sim}
\mathbb P=\mathcal N(\mu_p,\sigma^2\mathbb I_d)$;
target mean $\mu_q$;
shared scale $\sigma$
\ENSURE target sample
$Y\sim\mathbb Q=\mathcal N(\mu_q,\sigma^2\mathbb I_d)$;
acceptance indicator
$\texttt{accepted}\in\{\texttt{true},\texttt{false}\}$

\FOR{$k=1$ \TO $K$}
    % \STATE Obtain $(\widehat S_k,Z_{\perp,k},\delta,e) 
    % \gets
    % \texttt{Rank-1 Projector}
    % (\widehat Y_k,\mu_p,\mu_q,\sigma)$
    \STATE Obtain $(\widehat S_k, Z_\perp,\delta,e)$ from \eqref{eq:r1proj-variables} and \eqref{eq:rank1-projector}
\COMMENT{rank-1 projection}
\ENDFOR

\STATE $\lambda_0\gets 0$, $S_1\gets 1$

\FOR{$k=1$ \TO $K$}
    \STATE $\rho_k
    \gets
    \dfrac{\phi(\widehat S_k-\delta)}
          {\phi(\widehat S_k)}, \qquad \beta_k
    \gets
    1\wedge
    \dfrac{(\rho_k-\lambda_{k-1})_+}{S_k}$
    % \COMMENT{target-to-proposal likelihood ratio}

    % \STATE $\beta_k
    % \gets
    % 1\wedge
    % \dfrac{(\rho_k-\lambda_{k-1})_+}{S_k}$
    % \COMMENT{greedy acceptance probability}

    \STATE $U_k\sim\mathrm{Uniform}[0,1]$

    \IF{$U_k\leq\beta_k$}
        %\STATE $S\gets\widehat S_k$
        \STATE $Y\gets
        \mu_p+\sigma(\widehat S_k e+Z_{\perp,k})$
        \STATE $\texttt{accepted}\gets\texttt{true}$
        \RETURN $Y,\;\texttt{accepted},\; k$
    \ENDIF

    \STATE $\lambda_k\gets\lambda_{k-1}+G_k$
    \STATE $\mathcal H_k
    \gets
    \{s\in\mathbb R:
    \rho(s)\geq\lambda_k\}$
    \STATE $G_{k+1}
    \gets
    \mathbb Q(\mathcal H_k)
    -
    \lambda_k\mathbb P(\mathcal H_k)$
    \COMMENT{remaining target mass}
\ENDFOR

\STATE $S\sim r_{K+1}$ defined in~\eqref{eq:GRS_residual}
% $r_{K+1}(s)
% =
% (q(s)-\lambda_Kp(s))_+/S_{K+1}$
% \COMMENT{sample from residual}

\STATE $Y\gets
\mu_p+\sigma(S e+Z_{\perp,1})$
\STATE $\texttt{accepted}\gets\texttt{false}$
\RETURN $Y,\;\texttt{accepted},\; K+1$
\end{algorithmic}
\end{algorithm}

\textit{Acceptance:}
The introduction of a tree topology fundamentally changes the acceptance problem. In contrast to RMC, where each proposed state is individually coupled with the target through a maximal coupling, tree-based speculative sampling requires selecting one state from a \emph{list of proposals} and coupling the selected state with the target. This naturally gives rise to a \emph{list-coupling} problem. More precisely, maximal coupling minimizes the Hamming cost $\mathds{1}\{\widehat Y\neq Y\}$ between a single proposal $\widehat Y\sim\mathbb P$ and a target sample $Y\sim\mathbb Q$. At each node of our draft tree, however, we are given $K$ conditionally i.i.d.\ proposals $\widehat Y_1,\ldots,\widehat Y_K\sim\mathbb P$ and seek a coupling that minimizes the membership cost $\mathds{1}\bigl\{Y\notin\{\widehat Y_1,\ldots,\widehat Y_K\}\bigr\}$.
The optimal transport formulation of~\citet{sun2023spectr} characterizes the coupling that minimizes this cost. However, computing the optimal transport plan requires solving a linear program over the target domain, whose complexity grows exponentially with the number of proposals $K$. This is already costly for finite discrete state spaces and does not admit a tractable direct extension to continuous diffusion transitions. Computationally efficient approaches to list coupling have been studied in the context of language models; we provide a survey in the appendix. These methods, however, do not remain computationally tractable for continuous diffusion models.

As a workaround, we draw inspiration from relative entropy coding (REC)~\citep{li2024channel, flamich2026stochastic}, which is closely connected to speculative sampling, as elaborated below.
REC is a class of algorithms that tackles a \emph{remote sampling} problem:
a sender, which has access to a target distribution $\mathbb{Q}$ and a proposal distribution $\mathbb{P}$, communicates an index to a receiver, which shares $\mathbb{P}$, allowing the receiver to reproduce an \emph{exact sample} from $\mathbb{Q}$.
The fundamental communication cost of the index is governed, up to lower-order terms, by the relative entropy $D_{\mathrm{KL}}(\mathbb{Q}\|\mathbb{P})$, hence the name REC.
Correspondingly, REC implementations, which can be based on Poisson processes~\citep{li2018strong, liu2024universal}, dithered quantization~\citep{zamir1992universal}, or rejection sampling~\citep{harsha2010communication, flamich2023adaptive, flamich2024greedy}, need to examine $\tilde \Omega(\exp(D_{\mathrm{KL}}(\mathbb{Q}\|\mathbb{P})))$ shared proposals before identifying the communicated sample.
This viewpoint parallels speculative sampling: both problems seek to couple proposal samples with an exact target sample while maximizing the probability that the target can be recovered from the available proposals. While RMC optimizes the single-proposal overlap, measured by the total variation distance, REC algorithms construct exact couplings by processing an {\em ordered sequence} of proposal samples. Motivated by this connection, we replace the unrestricted list-coupling problem with an order-preserving \emph{sequence-coupling} problem. In list coupling, the $K$ proposals may be jointly matched to the target without regard to their ordering; in sequence coupling, proposals are examined in their sampled order, and the first proposal satisfying the acceptance rule is selected.

We then adapt the GRS procedure underlying the communication protocol of~\citet{harsha2010communication} to continuous diffusion transitions and embed it within our draft-tree topology. At each node, GRS sequentially examines the $K$ verified children, greedily assigns as much remaining target mass as possible to each proposal, and accepts the first proposal selected by the resulting coupling. If none of the $K$ children is accepted, the algorithm samples from the remaining target residual and terminates the current speculative path. More specifically, GRS maintains a sequence of scalar levels
$\{\lambda_k\}_{k=0}^K$, together with residual masses
$\{G_k\}_{k=1}^{K+1}$. The level $\lambda_k$ records how much of the
target-to-proposal likelihood ratio has already been assigned to the first
$k$ proposals, while $G_{k+1}$ denotes the target mass that remains
unassigned after round $k$. Initially, no target mass has been assigned, so
$\lambda_0=0$ and $G_1=1$. Let $\rho(s):=\frac{\mathrm d\mathbb Q}{\mathrm d\mathbb P}(s)=\frac{\phi(s-\delta)}{\phi(s)}$
% \[
% \rho(s)
% :=
% \frac{\mathrm d\mathbb Q}{\mathrm d\mathbb P}(s)
% =
% \frac{\phi(s-\delta)}{\phi(s)}
% \]
denote the target-to-proposal likelihood ratio in the one-dimensional
coordinate. At round $k$, GRS considers only the portion of
$\rho(\widehat S_k)$ that remains above the previously assigned level
$\lambda_{k-1}$. The proposal is therefore accepted with probability 
% $\beta_k=1\wedge\frac{\bigl(\rho(\widehat S_k)-\lambda_{k-1}\bigr)_+}{S_k}$.
\[
\beta_k
:=
1\wedge
\frac{\bigl(\rho(\widehat S_k)-\lambda_{k-1}\bigr)_+}{S_k}\ .
\]
This greedily assigns the largest admissible slice of the remaining target
mass to the current proposal while ensuring that the acceptance probability
does not exceed one. If the proposal is rejected, the assigned level is
increased to $\lambda_k=\lambda_{k-1}+G_k$, which induces the super-level set
\[
\mathcal H_k
=
\{s\in\mathbb R:\rho(s)\geq\lambda_k\}\ .
\]
The remaining target mass is then $G_{k+1} = \mathbb Q(\mathcal H_k) -\lambda_k\mathbb{P}(\mathcal H_k)$.
% \[
% S_{k+1}
% =
% \mathbb Q(\mathcal H_k)
% -
% \lambda_k\mathbb P(\mathcal H_k)
% =
% \int_{\mathbb R}
% \bigl(q(s)-\lambda_k p(s)\bigr)_+
% \,\diff s\ .
% \]
After $K$ unsuccessful rounds, GRS samples from the normalized residual
measure with density
\begin{align}
\label{eq:GRS_residual}
r_{K+1}(s)
=
\frac{\bigl(q(s)-\lambda_Kp(s)\bigr)_+}{G_{K+1}}\ ,
\end{align}
thereby preserving exactness. The D-GRS procedure is summarized in
Algorithm~\ref{alg:grs}. 

\section{Performance Analysis}
\label{sec:performance_analysis}
In this section, we provide performance analyses of the proposed D-GRS algorithm, and compare it to the RMC baseline. Specifically, we show the {\em exactness} of D-GRS and compute its
{\em acceptance probability}. The complete proofs are deferred to Appendix~\ref{sec:perf_analysis_appendix}. 
%For a target measure $\mathbb{Q}=\mathcal N \left( m^q,\sigma^2\mathbb I_d \right)$ and a proposal measure $\mathbb{P}= \mathcal N \left( m^p,\sigma^2\mathbb I_d \right)$, we denote $\rho:=\mathrm \diff\mathbb{Q}/\diff\mathbb{P}$ as their likelihood ratio. 
% Consider Algorithm~\ref{alg:grs} with $X_1,\cdots,X_K\stackrel{\mathrm{i.i.d.}}{\sim}\mathbb{P}$. If one of the first $M$ proposals is accepted, the algorithm returns that proposal. If all $M$ proposals are rejected, it samples from $\rho_M$. 
% It can be shown that the algorithm output $Y$ satisfies $Y\sim Q$, i.e., the output distribution is {\em {exact}}.

\begin{theorem}[Exactness]
\label{thm:exactness}
The output $Y$ of Algorithm~\ref{alg:grs} with target distribution $\mathbb{Q}=\mathcal{N}(\mu_q;\sigma^2\mathbb{I}_d)$ satisfies $Y\sim\mathbb{Q}$. 
\end{theorem}

The proof follows from the residual-measure characterization of GRS
\citep[Lemmas~III.1-2]{flamich2023adaptive}: one can explicitly calculate the target mass generated by the first $K$ proposal branches, and the amount contributed by a sample from its normalized residual distribution. The proposal and residual branches together recover the target distribution $\mathbb{Q}$ exactly. Next, we explicitly characterize the probability that Algorithm~\ref{alg:grs} accepts at least one drafted child.
\begin{theorem}[Acceptance Probability]
\label{thm:acceptance_probability}
Denote $\delta:=\frac{\left\|m^q-m^p\right\|_2}{\sigma}$. For Algorithm~\ref{alg:grs}, the probability of accepting at least one of the $K$ proposals is given by
\begin{equation}
\mathbb P(\exists k\in\{1,\cdots,K\}: Y=\widehat Y_k)
=
1-G_{K+1}\ .
\label{eq::grs_acceptance_probability}
\end{equation}
Furthermore, for any $\delta\in\mathbb{R}_+$, we have
\begin{align}
    G_{K+1}=\overline\Phi\left(\frac{1}{\delta}\ln \lambda_K - \frac{\delta}{2}\right) - \lambda_K \overline\Phi\left(\frac{1}{\delta}\ln \lambda_K + \frac{\delta}{2}\right),
\end{align}
where $\overline{\Phi}(x)$ is the standard Gaussian survival function. 
\end{theorem}
Note that setting $K=1$ and $\delta\in\mathbb{R}_+$, we obtain
\begin{equation}
\mathbb P(\exists k\in\{1,\cdots,K\}: Y=\widehat Y_k) = 2\overline{\Phi}\left(
\frac{\delta}{2}
\right)\ ,
\label{eq::single_proposal_acceptance}
\end{equation}
the {\em optimal} acceptance probability for maximal coupling which is achieved by RMC~\citep[Proposition $3.1$]{de2025accelerated}.
% Since $S_j=\mathbb P(J\geq j)$, the probability that all first $M$ proposals are rejected is $S_{M+1}$, which gives~\eqref{eq::grs_acceptance_probability}. 
% For the Gaussian case, the likelihood ratio depends only on a one-dimensional projection along $m^q-m^p$; evaluating the corresponding Gaussian tail probabilities yields the recursion in~\eqref{eq::grs_survival_recursion}. When $M=1$, this reduces to the standard reflection maximal-coupling acceptance probability in~\eqref{eq::single_proposal_acceptance}, which coincides with \citep[Proposition 3.1]{de2025accelerated}. 
% The acceptance probabilities for different number of drafters are shown in Figure~\ref{fig:accept_prob}, which shows how the acceptance probability improves upon RMC with increasing $K$. 
% \begin{figure}[t]
%   \centering
%   \includegraphics[width=0.75\linewidth]{Figures/curves/accept_prob.pdf}
%   \caption{Node acceptance probability for different number of drafters. When $K=1$ }
%   \label{fig:accept_prob}
% \end{figure}
The per-step acceptance probabilities for different value of $K$ are plotted in Figure~\ref{fig:accept_prob} in Appendix \ref{sec:acceptance_prob_appendix}.

\section{Experiments}
\label{sec:experiments}

% We evaluate D-GRS against RMC in three different settings: diffusion sampling from Gaussian mixture target with analytically tractable scores, pixel-space conditional and unconditional diffusion sampling over CIFAR-10 ($32\times32$)~\citep{cifar10} and FFHQ ($64\times64$)~\citep{ffhq}, respectively, and text-to-image latent space sampling using Stable Diffusion 3.5 \citep{sd3}.
We evaluate the performance of D-GRS relative to RMC across three generative tasks: (i) diffusion sampling from a Gaussian mixture target utilizing analytically tractable scores; (ii) pixel-space generation, including conditional sampling on CIFAR-10 ($32\times32$)~\citep{cifar10} and unconditional sampling on FFHQ ($64\times64$)~\citep{ffhq} datasets; and (iii) text-to-image latent space sampling via Stable Diffusion 3 (SD3) \citep{esser2024scaling}.

\paragraph{Protocol.}
All methods are evaluated under the \emph{self-speculative} regime of \citep{de2025accelerated} for $B=|\Verts|$. Specifically, the draft is constructed directly from the target denoiser by freezing the velocity field computed at the root of the speculation tree (or the window head for RMC) and reusing it throughout the draft trajectory. Sampling follows the Euler-Maruyama discretization of the linear probability path $x_\sigma = (1-\sigma)x_0 + \sigma\,\xi$, $\xi\sim\mathcal{N}(0,\mathbb{I}_d)$, with
stochasticity level (churn) $\varepsilon\in (0,1]$\footnote{$\varepsilon=0$ recovers the
deterministic flow sampler, rendering speculation vacuous as the acceptance probability drops to zero.}.

\paragraph{Cost metric and matched budgets.}
Following \citep{de2025accelerated} we measure the computational cost in NFEs.
% \textcolor{blue}{\emph{target-model function evaluations} (NFEs) [Already defined!]}. A single evaluation is defined as one forward pass through the target network for a batch of states, independent of the batch size. This appropriately reflects the constraints of the latency-bound regime that speculative sampling is designed to address.
To compare D-GRS and RMC at equal cost per target call, we evaluate every D-GRS configuration $(K,L)$ against an RMC chain of length $B$. Both methods therefore process an identical verification batch of size $B$ per round, differing only in structural allocation. Consequently, the proposal budget $B$ serves as the natural $x$-axis throughout our comparisons. All costs are reported as expected NFE per trajectory/image. 
%representing the invariant, standalone cost of generating a single sample.

%============================================================================
\subsection{Gaussian Mixture Target}
\label{sec:exp-gm}

\paragraph{Setting.}
We reproduce the low-dimensional setting of \citep[App.~I.2]{de2025accelerated}. The target is a mixture of $5$
isotropic Gaussians in $\mathbb{R}^{d}, d=\gmdim$, with $\mu_i \sim \mathcal{U}\!\left([-2,2]^d\right)$ and $\sigma_i \sim \mathcal{U}\!\left([0.10,\,0.25]\right)$
% \begin{equation}
%   p_0 \;=\; \frac{1}{5}\sum_{c=1}^{15}\mathcal{N}\!\left(\mu_c,\ s_c^2 I_d\right),
%   \quad
%   \mu_c \sim \mathcal{U}\!\left([-2,2]^d\right),\quad
%   s_c \sim \mathcal{U}\!\left([0.10,\,0.25]\right),
% \end{equation}
drawn once from a fixed seed
and then held fixed across all runs. We use the \emph{exact} velocity field which, under the linear probability path, is available in closed form. 
We use $T=30$ generation steps with stochasticity $\varepsilon=\gmchurn$. We sweep the full grid $L\in\{2,\dots,7\}\times K\in\{1,\dots,7\}$, %which spans budgets from $B=1$ to $\approx 9.6\times10^{5}$, 
and run $100$ independent trajectories per entry. Speedups are reported relative to the reference cost of the standard sampler, $(K,L)=(1,1)$,
which spends $30$ target calls per trajectory.

\paragraph{Results.}
Figure~\ref{fig:gm-heatmaps} shows the expected speedups across the $(K,L)$ grid and corresponding budgets $B$. RMC speedup plateaus ($\approx 1.96\times$--$2.00\times$) across six orders of magnitude in $B$, bottlenecked by the geometric decay of its linear sequential look-ahead.
Conversely, D-GRS allocates budget to $K$ proposals per node to boost per-step acceptance, avoiding the rapid saturation of look-ahead scaling. For $L \ge 3$, D-GRS strictly dominates, yielding a $2.33\times$ speedup at $(6,6)$ compared to RMC's $1.99\times$. Overall, D-GRS reduces target calls by 8--9\% ($2.14\times$ vs.\ $1.96\times$ at $B=155$) without compromising exactness.
% Overall, because dimensionality $d$, step size $\gamma$, and stochasticity $\varepsilon$ jointly widen $\delta/\sigma$. The
% wider the $\delta/\sigma$, the lower the per-step acceptance probability, and the more each additional child is worth. 
%This shifts the D-GRS crossover toward lower budgets. 

\begin{figure}[t]
  \centering
  \includegraphics[width=0.9\linewidth]{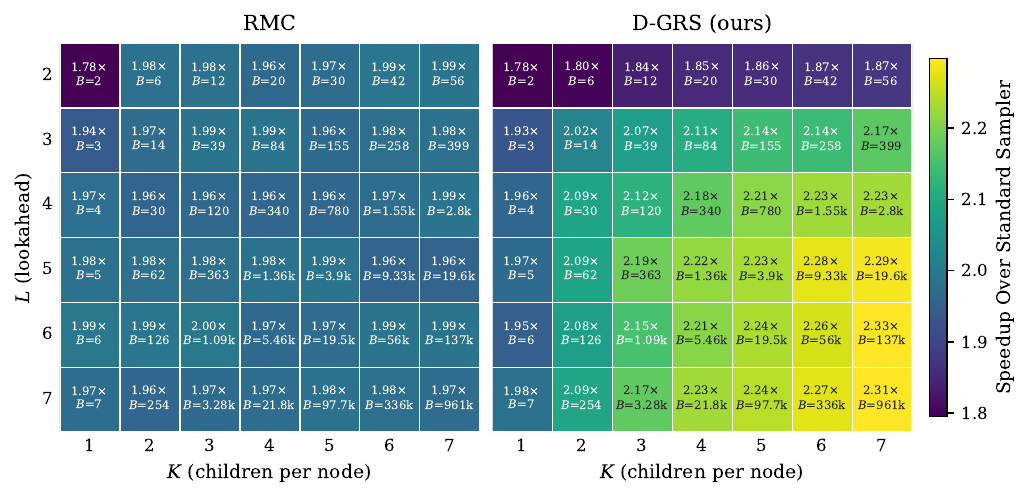}
  \caption{Expected speedup in the Gaussian mixture setting over $100$ independent trajectories per $(K,L)$ pairs. Performance for RMC (left) remains uniform across the grid, whereas D-GRS (right) exhibits steady improvement as $B$ is increasingly allocated to width $K$.}
  \label{fig:gm-heatmaps}
\end{figure}

% ============================================================================
\subsection{Pixel-Space Diffusion}
\label{sec:exp-pixel-space}
For both conditional and unconditional diffusion sampling over the pixel-space, we adopt pretrained denoisers provided by \citet{karras2022elucidating}. At time step $t$, the model $D_\theta(x_t;\varsigma) = \mathbb{E}[x_0 \mid x_t = x_0 + \varsigma\xi]$ estimates the uncorrupted data $x_0$ from a noisy observation $x_t$, perturbed by standard Gaussian noise $\xi$ at a noise scale $\varsigma$. To integrate this denoiser into our setting, we recast it as a velocity field $v_\theta(x_t, t)$ along a linear interpolant. By applying $\varsigma = t/(1-t)$, the velocity field evaluates to $v_\theta(x_t,t) = \frac{1}{t}x_t - D_\theta\left(\frac{x_t}{1-t}; \frac{t}{1-t}\right)$.
% \begin{equation*}
%     v_\theta(x_t,t) = \frac{x_t - D_\theta\left(\frac{x_t}{1-t}; \frac{t}{1-t}\right)}{t}.
% \end{equation*}
Finally, to ensure numerical stability during generation, the evaluated noise scale is clipped to $\varsigma \in [10^{-3}, 80]$, which remains strictly within the original model's native sampling range.

\begin{figure}[t]
    \centering
    \includegraphics[width=0.8\linewidth]{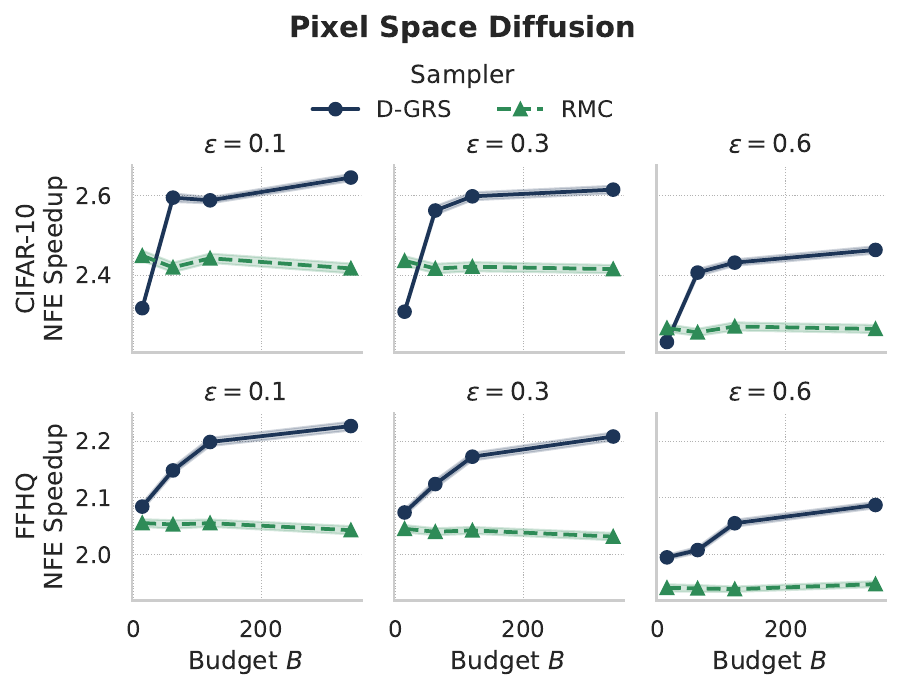}
    \caption{NFE speed-up relative to proposal budget $B$ for CIFAR-10 (top) and FFHQ (bottom) at varying stochasticity $\epsilon \in \{0.1, 0.3, 0.6\}$.}
    \label{fig:pixel-space-results}
\end{figure}

\paragraph{Experimental Setup.} Each sampler generates $100$ images over $100$ time steps. For CIFAR-10, the generation is conditioned on labels sampled uniformly at random from the ten available classes. We evaluate four configuration pairs $(K,L) \in \{(2,3),\allowbreak (2,5),\allowbreak (3,4),\allowbreak (4,4)\}$, which correspond to computational budgets $B \in \{14, 62, 120, 340\}$, with RMC executed on the matched windows $(1, B)$. This procedure is repeated across three stochasticity levels, $\varepsilon \in \{0.1, 0.3, 0.6\}$. Figure~\ref{fig:pixel-space-results} illustrates the resulting mean and standard deviation of the speedups relative to the baseline standard sampler, which incurs a fixed cost of $100$ NFEs per image.

\paragraph{Results.} In Figure~\ref{fig:pixel-space-results}, we report the NFE speedups per image against the computational budget $B$ for two datasets with distinct spatial resolutions: CIFAR-10 ($32 \times 32$, top row) and FFHQ ($64 \times 64$, bottom row). Across both tasks, increasing the churn level degrades the sampling efficiency of both methods. Specifically, higher stochasticity increases the divergence between the draft and target transition kernels, widening the gap between their respective means, thereby attenuating the sampling acceptance rates across both coupling mechanisms.

On CIFAR-10, RMC exhibits a marginal empirical advantage at the minimal budget ($B=14$), where its linear sequential topology achieves greater effective depth compared to the shallow branching of a tree. However, for $B\geq 62$ the efficiency trend reverses: RMC suffers from plateaued scaling, yielding a fixed speedup irrespective of $B$, whereas D-GRS efficiently leverages its speculative tree topology to maintain monotonically increasing speedups.

On the higher-dimensional FFHQ benchmark, absolute speedups are inherently reduced ($\approx 2.0$--$2.2\times$, down from $2.2$--$2.6\times$ on CIFAR-10), as the mean discrepancy between the draft and target kernels scales with $\sqrt{d}$. Crucially, the relative advantage of D-GRS becomes absolute in this higher dimensional regime, strictly dominating RMC at \emph{every} tested budget and churn level.
% To verify exactness, we evaluate FID on 50k samples ($K=2, L=3, \varepsilon=0.3$). On CIFAR-10, D-GRS (1.824) and RMC (1.822) closely match the 1.79 baseline \citep{karras2022elucidating}. On FFHQ, D-GRS (2.86) and RMC (2.88) perform slightly above the 2.39 baseline \citep{karras2022elucidating}.
To empirically verify exactness, we evaluate FID on 50,000 samples generated with $K=2$, $L=3$, and $\varepsilon=0.3$. On CIFAR-10, D-GRS and RMC achieve FID scores of 1.824 and 1.822, respectively, closely tracking the 1.79 baseline from \citet{karras2022elucidating}. On FFHQ, D-GRS and RMC yield respective scores of 2.86 and 2.88, performing comparably to, though slightly above, the 2.39 baseline \citep{karras2022elucidating}.

% \begin{figure}[t]
%     \centering
%     \begin{minipage}[t]{0.48\linewidth}
%         \centering
%         \includegraphics[width=0.91\linewidth]{Figures/experiments/pixel-space/speedup-pixel-space.pdf}
%         \caption{NFE speed-up relative to proposal budget $B$ for CIFAR-10 (top) and FFHQ (bottom) at varying stochasticity $\epsilon \in \{0.1, 0.3, 0.6\}$.}
%         \label{fig:pixel-space-results}
%     \end{minipage}\hfill
%     \begin{minipage}[t]{0.48\linewidth}
%         \centering
%         \includegraphics[width=0.7\linewidth]{Figures/experiments/latent-space/speedup-sd3.pdf}
%         \caption{Expected NFE speed-up over the vanilla target sampler as a function of $K$ ($512\times512$ resolution, $\varepsilon=0.8$, $L=2$).}
%         \label{fig:sd3-speedups}
%     \end{minipage}
% \end{figure}

\begin{figure}[t]
    \centering
        \includegraphics[width=0.5\linewidth]{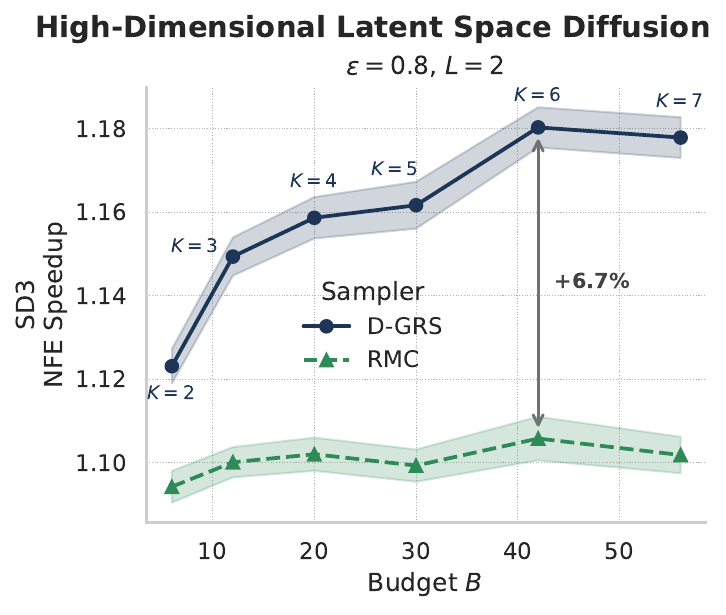}
        \caption{Expected NFE speed-up over the vanilla target sampler as a function of $K$ ($512\times512$ resolution, $\varepsilon=0.8$, $L=2$).}
        \label{fig:sd3-speedups}
\end{figure}

\begin{table}[t]
\footnotesize
\setlength{\tabcolsep}{4pt}
\centering
\begin{tabular}{lcccccc}
\toprule
& \multicolumn{6}{c}{$K\quad$ - $\quad$ Target score: $ 26.90 \pm 0.37$} \\ \cmidrule(lr){2-7}
\textbf{Method} & 2 & 3 & 4 & 5 & 6 & 7 \\
\midrule
%Target & \multicolumn{6}{c}{$26.90 \pm 0.37$} \\

D-GRS & $26.78 \pm 0.34$ & $26.32 \pm 0.34$ & $26.67 \pm 0.34$ & $26.72 \pm 0.35$ & $26.75 \pm 0.33$ & $26.65 \pm 0.35$ \\
RMC & $26.59 \pm 0.35$ & $26.74 \pm 0.33$ & $26.48 \pm 0.34$ & $26.77 \pm 0.34$ & $26.39 \pm 0.37$ & $26.87 \pm 0.33$ \\
\bottomrule
\end{tabular}
\caption{CLIP scores (mean $\pm$ standard error of the mean over $n=100$ prompts) for $L=2$, comparing D-GRS and RMC across different $K$ values against the target baseline.}
\label{tab:clip_scores}
\end{table}

% Scaling Efficiency in Pixel-Space Diffusion
% \begin{figure}[t]
%     \centering
%     \includegraphics[width=.9\linewidth]{Figures/experiments/pixel-space/speedup-pixel-space.pdf}
%     \caption{NFE speed-up relative to the proposal budget $B$ for CIFAR-10 (top) and FFHQ (bottom) at varying stochasticity levels $\epsilon \in \{0.1, 0.3, 0.6\}$.}
%     \label{fig:pixel-space-results}
% \end{figure}

% ============================================================================
\subsection{High-dimensional Latent Space Diffusion}
\label{sec:exp-sd3}
Stable Diffusion 3.5 (SD3)\footnote{The model is publicly available at Huggingface as \texttt{stabilityai/stable-diffusion-3.5-medium}.} is natively trained via rectified flow. It therefore 
%directly outputs the velocity field $v_\theta(y_n, t_n)$ along the linear probability path defined in Section~\ref{sec:experiments}.
relies on ODE integration, with transitions being deterministic point masses. Consequently, draft and target proposals exhibit a total variation distance of one, rendering speculation vacuous. We resolve this by sampling from a marginal-preserving SDE family along the identical linear probability path, requiring no retraining. The score is computed directly from the velocity as $s_\theta = -(x_\sigma + (1-\sigma) v_\theta)/\sigma$. Adding the drift $-\frac{1}{2}\varepsilon^2 g^2 s_\theta$ (where $g^2(\sigma) = 2\sigma/(1-\sigma)$) and noise of scale $\varepsilon g$ ensures every marginal $p_\sigma$ remains invariant. Under an Euler--Maruyama discretization, each step yields a Gaussian kernel $\mathcal{N}(m, \varsigma^2 I)$. Since the standard deviation $\varsigma = \varepsilon g \sqrt{-\gamma}$ is purely a function of the time schedule, draft and target kernels diverge \emph{only by their means}---the exact mathematical structure required by our coupling framework. A detailed derivation is deferred to Appendix~\ref{sec:sd3_conversion}. 

\paragraph{Experimental Setup.}
We evaluate high-resolution generation at $512\times512$, utilizing $T=50$ generation steps, a free-guidance scale of $7.0$, and a stochasticity level of $\varepsilon=0.8$. To adapt the native resolution to ours, we adopt a timestep shift of $3.0$. Throughout the simulations, we fix the lookahead at $L = 2$ and sweeping the branching factor $K \in \{2,3,4,5,6,7\}$, yielding budgets $B \in \{6, 12, 20, 30, 42, 56\}$. Generative sampling is executed within the VAE latent space of dimension $d=16 \times 64 \times 64 = 65{,}536$. 
All configurations generate $100$ images conditioned on an identical, pre-sampled set of $100$ prompts drawn uniformly without replacement from the COCO 2014 validation set \citep{coco}. %strictly isolating algorithmic effects from prompt variability across all parameter sweeps.

\paragraph{Results.} Figure~\ref{fig:sd3-speedups} illustrates the empirical speed-up over the standard target sampler as a function of $K$ and a constrained lookahead depth of $L=2$. As the number of proposals per node ($K$) scales from $2$ to $7$, the corresponding computational budget $B$ expands from $6$ to $56$. Consistent with prior topological analyses, the proposed method effectively translates this expanded budget into measurable acceleration. By allocating compute toward $K$, D-GRS yields a monotonically increasing speedup that peaks near $1.18\times$ at $K=6$ ($B=42$), with an increment of $6.7\%$ relative to RMC. In contrast, RMC fails to meaningfully exploit the additional budget. Bottlenecked by the geometric decay of its sequential acceptance, the reflection sampler exhibits severe saturation, plateauing between $1.10\times$ and $1.12\times$ across the entire parameter sweep. This widening performance gap clearly demonstrates that scaling multiple proposals per node ($K$) successfully improves per-step acceptance rates, allowing the branching topology to strictly dominate the linear chain in high-dimensional text-to-image synthesis.
To empirically show the exactness of both sampling procedures, we report in Table~\ref{tab:clip_scores} the CLIP scores \citep{clip} for each budget configuration, computed using \texttt{openai/clip-vit-large-patch14}.

\section{Conclusions and Future Work}
We identify a key topological limitation of reflection maximal coupling (RMC), which allocates its speculative budget to a single draft chain. To better utilize parallel computation, we introduce draft trees and propose D-GRS, which adapts greedy rejection sampling from the relative entropy coding literature to the resulting list-coupling problem. Across our experiments, D-GRS improves NFE speed-up over RMC by up to $8.3\%$ in practical settings.

This work is a first step toward more efficient speculative diffusion sampling. Important open directions include overcoming the rapid degradation of exact coupling efficiency with generation dimension and adapting the draft topology online to the proposal quality and available compute budget.

\clearpage
\bibliographystyle{plainnat}
\bibliography{references}
\clearpage

\newpage

% \section{Changes to introduction}
% A major drawback of diffusion models, however, is their high inference cost. Image generation requires repeatedly evaluating a large neural network over tens to hundreds of denoising iterations, making sampling substantially slower. This computational bottleneck has motivated a broad range of acceleration techniques. Existing approaches include distillation-based methods~\citep{salimans2022progressive,meng2023distillation,song2023consistency}, which train a smaller student model, fast samplers that reduce the number of denoising steps~\citep{zhang2022fast}, flow maps~\citep{boffi2026build, geng2026mean, geng2026improved} for few-step generation, and stochastic flow maps~\citep{holderrieth2026diamond,potaptchik2026meta,mccallum2026strong} for few-step controlled generative sampling. While effective in reducing latency, these approaches typically incur either additional training cost or degradation in sample quality~\citep{dieleman2024paradox,karras2022elucidating}.
\appendix

\section{Literature Review}

\subsection{Diffusion Models}

Diffusion models implicitly learn data distributions from empirical samples. The generative process is initialized with a sample from a tractable prior, typically a Gaussian distribution, and iteratively denoises it to recover a sample from the target data distribution~\citep{lai2025principles}. This mechanism is formalized by defining a fixed forward process that progressively corrupts the data with noise, coupled with a parameterized neural network trained to reverse this corruption. Several theoretical frameworks formulate this concept, including the Denoising Diffusion Probabilistic Model (DDPM) \citep{sohl2015deep, ho2020denoising} and noise-conditioned score networks \citep{song2019generative}. These were subsequently unified as discretizations of a single continuous-time stochastic differential equation (SDE) \citep{song2021scorebased}, alongside the more recent flow matching paradigm \citep{lipman2022flow, liu2023flow}. Fundamentally, these formulations are mathematically equivalent up to a transformation of the model's output parameterization and noise schedule~\citep{kingma2021variational, lai2025principles}. A critical practical consequence of this equivalence is that the training objective does not uniquely dictate the sampling algorithm: models trained via denoising can be sampled deterministically using a probability flow ordinary differential equation (ODE) \citep{song2021scorebased}, while models trained under a flow matching objective can employ stochastic sampling schemes \citep{liu2023flow, Hu_2025_CVPR}.
These generative frameworks have established state-of-the-art performance across a diverse array of modalities, encompassing image synthesis~\citep{dhariwal2021diffusion, rombach2022high, karras2022elucidating}, video generation~\citep{ho2022video, bartal2024}, protein design~\citep{Watson2023}, prediction of biomolecular structure~\citep{Abramson2024}, and numerous other applications~\citep{yang2023diffusion}.

Despite their empirical success, a primary limitation of diffusion models is the iterative nature of the sampling process. Efforts to accelerate inference by reducing the number of integration steps have been extensively investigated and generally branch into training-free and training-based methodologies~\citep{dieleman2024paradox, dieleman2026flowmaps, lai2025principles}. Training-free approaches leverage the fixed, pre-trained model and achieve acceleration through the deployment of advanced ODE and SDE solvers that allow for larger step sizes without compromising accuracy~\citep{song2021denoising, zhang2022fast, lu2022dpm, lu2025dpm}. Conversely, training-based methods explicitly optimize a model to execute large synthesis steps. These approaches vary along two primary axes. The first is the source of supervision: they may distill knowledge from a pretrained teacher model~\citep{luhman2021knowledge, salimans2022progressive, song2023consistency, sabour2026align} or train a few-step model entirely from scratch~\citep{song2023consistency, frans2025one, geng2026mean, boffi2026build}. The second axis concerns the matching objective: certain techniques enforce the preservation of trajectories connecting noise to data~\citep{salimans2022progressive, song2023consistency, boffi2025flow}, whereas others constrain the final model output to match the target marginal distribution~\citep{Yin_2024_CVPR, yin2024improved, sauer2024adversarial}. While initially developed for ODE trajectories, these acceleration techniques have recently been generalized to SDEs. This includes methods that learn transition kernels to replicate the process's marginal distributions~\citep{potaptchik2026meta, holderrieth2026diamond}, as well as those that learn the solution map pathwise, conditioned on specific realizations of the driving noise~\citep{mccallum2026strong}.

\subsection{Speculative Sampling for Language Models}
Standard autoregressive decoding approaches, such as top-$k$ sampling~\citep{radford2019language, fan2018hierarchical}, nucleus sampling~\citep{holtzman2020curious}, and the permute-and-flip decoder~\citep{zhao2024permute}, generate tokens one at a time and therefore incur the inherent latency of one target-model call for each token. 
Speculative sampling~\citep{chen2023accelerating, leviathan2023fast} alleviates this bottleneck by employing a smaller draft model to propose multiple future tokens, while the larger target model verifies these proposed continuations in parallel, with a rejection-correction step ensuring that the exact target-model distribution is preserved.
The resulting efficiency gain depends on how closely the draft distribution approximates the target distribution: more accurate proposals yield higher acceptance rates and longer verified blocks.
Recent works have further advanced this paradigm through optimal-transport-based token selection~\citep{sun2023spectr}, tree-based inference and verification~\citep{miao2024specinfer}, and multi-draft architectures~\citep{khisti2024multi, hu2025towards, lin2026spectr}. 
Multi-draft speculative decoding can be viewed as an application of list-coupling schemes~\citep{rowan2026list}, which strictly extend ordinary coupling for speculative decoding~\citep{daliri2025coupling}.
Recently, speculative sampling has been used to accelerating the diffusion modesl~\citep{de2025accelerated, soen2026accelerating}.

\subsection{Relative Entropy Coding}

The sequential coupling methodology introduced here for speculative diffusion is theoretically grounded in greedy rejection sampling~\citep{harsha2010communication}. This technique falls within a broader family of remote coupling methods identified by various terms in the literature, including channel simulation~\citep{li2024channel}, channel synthesis~\citep{cuff2013distributed}, and reverse channel coding~\citep{bennett2002entanglement}. Hereafter, we refer to this class of schemes as relative entropy coding (REC)~\citep{flamich2026stochastic}.
Fundamentally, REC is formulated as a lossy compression problem designed to establish the minimal communication overhead required to simulate a target stochastic channel over a noiseless link. The nomenclature derives from the theoretical result that, in the regime of exact simulation, the fundamental lower bound on this communication cost is governed by a relative entropy measure.

In this work, we employ greedy rejection sampling~\citep{harsha2010communication}, a methodology whose capabilities have been significantly expanded in recent literature~\citep{flamich2023adaptive, flamich2024greedy, hill2026rejection}. A parallel and prominent class of REC frameworks relies on Poisson processes. Rooted in the Poisson functional representation~\citep{li2018strong}-a concept intimately connected to $\mathrm{A}^*$ sampling~\citep{maddison2014sampling, maddison2016poisson}-these schemes have found broad application across diverse domains. Notable examples include hypothesis testing~\citep{guo2024hypothesis}, network information theory~\citep{li2021unified, liu2025one, liu2025nonasymptotic}, and machine learning~\citep{he2024accelerating}. Related methodologies for approximate channel simulation have been proposed by~\citet{block2023sample, kobus2024gaussian, flamich2024some}. Nevertheless, since these approaches yield inexact simulations and therefore fall short of the theoretical relative entropy limit, they are precluded from the standard definition of REC.
% In this work, we utilize greedy rejection sampling~\citep{harsha2010communication}, whose scope has been further advanced by recent works~\citep{flamich2023adaptive, flamich2024greedy, hill2026rejection}. Another important class of REC schemes is based on Poisson processes, starting from the Poisson functional representation~\citep{li2018strong}, which is closely related to $\mathrm{A}^*$ sampling~\citep{maddison2014sampling, maddison2016poisson} and has recently been applied to a wide range of areas, including hypothesis testing~\citep{guo2024hypothesis}, network information theory~\citep{li2021unified, liu2025one, liu2025nonasymptotic}, and machine learning~\citep{he2024accelerating}.
% For related approximate channel simulation schemes, see also~\citep{block2023sample, kobus2024gaussian, flamich2024some}; however, because these schemes are not exact and hence do not achieve the relative entropy limit, they should not be referred to as REC schemes.

\section{On Performance Analysis}\label{sec:perf_analysis_appendix}

In this section, we provide proofs of the theoretical guarantees discussed in Section~\ref{sec:performance_analysis}.
For completeness, we restate the theorems before presenting their proofs.

Throughout this section $\mathbb P=\mathcal N(\mu_p,\sigma^2 I_d)$ and
$\mathbb Q=\mathcal N(\mu_q,\sigma^2 I_d)$ denote the proposal and target
distributions on $\Omega=\mathbb R^d$. Furthermore, let $J$ denote the index of the proposal accepted by Algorithm~\ref{alg:grs}, so that $Y=\widehat Y_J$. Let also $(\lambda_j)_{j\geq 0}$ be an increasing sequence of levels that define superlevel sets of the density ratio $\rho=\frac{\diff \mathbb Q}{\diff \mathbb P}$, i.e.,
$$
\mathcal{H}_j := \left\{y\in\Omega:\rho(y)\geq \lambda_j \right\}.
$$
Finally, $\overline\Phi(x):=1-\Phi(x)$
is the standard Gaussian survival function, and $(x)_+:=\max\{x,0\}$.

\subsection{Proof of Exactness}

\begin{theorem}[Exactness]
\label{thm:exactness}
The output $Y$ of Algorithm~\ref{alg:grs} with target distribution $\mathbb{Q}=\mathcal{N}(\mu_q;\sigma^2\mathbb{I}_d)$ satisfies $Y\sim\mathbb{Q}$. 
\end{theorem}

\begin{proof}
Let $\widehat Y_{1:K}\overset{\mathrm{i.i.d.}}{\sim}\mathbb P$ be $K$ proposals. 
By the residual-measure characterization of greedy rejection sampling~\citep[Lemma III.1-2]{flamich2023adaptive}, for every measurable set $A\subseteq\mathbb R^d$, the target mass that remains after the first $j$ proposals is
\begin{align}
\Delta_j(A)
 &=
\mathbb Q(A) - \mathbb P \left( J\leq j,\, \widehat Y_J\in A \right) \nonumber \\ &=\int_A \left(\rho(y)-\lambda_j\right)_+ \mathbb P(\mathrm dy). \label{eq::target_remain}
\end{align}
Note that if we take $A=\Omega$, we have 
\begin{align}\label{eq::survival}
\Delta_j(\Omega) = 1-\mathbb P(J\leq j) = P(J > j) := G_{j+1}. 
\end{align}
Hence when $G_{j+1}>0$, the normalized residual distribution is 
\begin{equation}\label{eq::normalized_residual}
r_j(A) = \frac{\Delta_j(A)}{G_{j+1}}.
\end{equation}
Therefore, by combining \eqref{eq::target_remain}, \eqref{eq::survival}, and \eqref{eq::normalized_residual}, for every measurable $A\subseteq\mathbb R^d$, we have 
\begin{align*}
\mathbb P(Y\in A) &= \mathbb P \left( J\leq j, \widehat Y_J\in A \right) + \mathbb P(J>j)\ r_j(A)\\
&=
\mathbb Q(A)-\Delta_j(A) + G_{j+1}
\frac{\Delta_j(A)}{G_{j+1}}\\
&=
\mathbb Q(A).
\end{align*}
Therefore $Y\sim \mathbb Q$ exactly. 
\end{proof}

\subsection{Proof of Acceptance Probability}\label{sec:acceptance_prob_appendix}
% \begin{definition}[Residual mass function]
% \label{def:g_delta}
% For $\delta, \lambda>0$ let $g_\delta:[0,\infty)\to[0,1]$ be given by
% \begin{equation}
%   g_\delta(\lambda)
%   :=
%   \overline\Phi\!\left(\frac{\ln\lambda}{\delta}-\frac{\delta}{2}\right)
%   -
%   \lambda\,\overline\Phi\!\left(\frac{\ln\lambda}{\delta}+\frac{\delta}{2}\right),
%   \label{eq::g_delta}
% \end{equation}
% with $g_\delta(0):=1$, which is the limiting value of \eqref{eq::g_delta} as
% $\lambda\downarrow 0$.
% \end{definition}

\begin{theorem}[Acceptance Probability]
Denote 
\begin{align*}
    \Delta:=(\mu_q-\mu_p)/\sigma, \quad \delta:=\|\Delta\|, \quad e:=\Delta/\delta.
\end{align*}
For Algorithm~\ref{alg:grs}, the probability of accepting at least one of the $K$ proposals is given by
\begin{equation}
\mathbb P(\exists k\in\{1,\cdots,K\}: Y=\widehat Y_k)
=
1-G_{K+1}\ .
\label{eq::grs_acceptance_probability}
\end{equation}
Furthermore, for any $\delta\in\mathbb{R}_+$, we have
\begin{align}
    G_{K+1}=\overline\Phi\left(\frac{1}{\delta}\ln \lambda_K - \frac{\delta}{2}\right) - \lambda_K \overline\Phi\left(\frac{1}{\delta}\ln \lambda_K + \frac{\delta}{2}\right),
\end{align}
where $\overline{\Phi}(x)=1-\Phi(x)$ is the standard Gaussian survival function. 
\end{theorem}

\begin{proof}
By Lemma~III.1 of \citet{flamich2023adaptive}, $G_j$ is the probability that
greedy rejection sampling reaches proposal $j$, i.e.\ $G_j=\Pr(J\ge j)$.
Consequently the probability that the first $j$ proposals are all rejected is
\begin{equation*}
  \Pr(J>j)=\Pr(J\ge j+1)=G_{j+1},
\end{equation*}
so $\Pr(J\le j)=1-G_{j+1}$, and therefore
\begin{align*}
  \Pr\big(\exists k\in\{1,\dots,K\}:\,Y=\widehat Y_k\big)
  &=\Pr(J\le K)\\
  &=1-G_{K+1}.
  %\label{eq::step_one}
\end{align*}
The remainder of the proof computes $G_{K+1}$. Combining \eqref{eq::target_remain} and \eqref{eq::survival}, we obtain
\begin{align}
G_{j+1}
= \Delta_j(\Omega)
= \int_\Omega \left(\rho(y)-\lambda_j\right)_+\,\mathbb P(\diff y).
\label{eq::residual_mass}
\end{align}
Since the integrand
in~\eqref{eq::residual_mass} vanishes outside the superlevel set
\begin{equation*}
\mathcal H_j := \left\{y\in\Omega:\ \rho(y)\geq \lambda_j \right\},
\end{equation*}
we have
\begin{equation}
G_{j+1} = \int_{\mathcal H_j}\left(\rho(y) \lambda_j\right)\mathbb P(\diff y) = \mathbb Q(\mathcal H_j)-\lambda_j\,\mathbb P(\mathcal H_j). \label{eq::grs_two_terms}
\end{equation}
Computing $G_{j+1}$ has thus been reduced to computing the $\mathbb P$- and
$\mathbb Q$-masses of a superlevel set of the density ratio. For Gaussian $\mathbb P$ and $\mathbb Q$ with a common covariance $\sigma^2 I_d$, the log target-to-proposal ratio becomes:
%both are univariate Gaussian tail probabilities, as we now show. 
\begin{align*}
  \ln\rho(y)
  &=\frac{\|y-\mu_p\|^2-\|y-\mu_q\|^2}{2\sigma^2}\\
  &=\frac{2(\mu_q-\mu_p)^\top y-\big(\|\mu_q\|^2-\|\mu_p\|^2\big)}{2\sigma^2}\\
  &=\frac{(\mu_q-\mu_p)^\top\big(y-\bar\mu\big)}{\sigma^2}
  =\frac{\delta\,e^\top(y-\bar\mu)}{\sigma},
  % =\delta\,s(y),
\end{align*}
with $\bar\mu:=\tfrac12(\mu_p+\mu_q)$.
Let
\begin{equation*}
s(y):= \frac{1}{\sigma}\,e^\top\!\left(y-\bar\mu \right),
\end{equation*}
then the likelihood ratio becomes a monotone function of the scalar statistic $s$,
\begin{equation*}
\rho(y) = \exp\left(\delta\, s(y)\right).
\end{equation*}
% Moreover, since $s$ is an affine functional of a Gaussian vector with
% $\|e\|=1$, so if $X\sim \mathbb P$ and $Y\sim \mathbb Q$, then

Because $\delta>0$, the map $t\mapsto e^{\delta t}$ is strictly increasing, and the superlevel set is the half-space
\begin{align*}
\mathcal H_j
=
\left\{y:\ \rho(y)\geq \lambda_j\right\}
=
\left\{y:\ s(y)\geq \tau_j\right\},
\quad
\tau_j=\frac{\ln \lambda_j}{\delta}.
\end{align*}
Now, to characterize the masses $\mathbb Q(\mathcal{H}_j)$ and $\mathbb P(\mathcal{H}_j)$, we are left with the characterization of the the laws of $s$ when $X\sim \mathbb P$ and $Y\sim \mathbb Q$. $s$ is an
affine functional of a Gaussian vector, hence Gaussian with means
\begin{align*}
  \mathbb E[s(X)]&=\frac{e^\top(\mu_p-\bar\mu)}{\sigma}
  =\frac{e^\top(\mu_p-\mu_q)}{2\sigma}
  =-\frac{\delta}{2},\\
  \mathbb E[s(Y)]&=\frac{e^\top(\mu_q-\bar\mu)}{\sigma}
  =\frac{e^\top(\mu_q-\mu_p)}{2\sigma}
  =+\frac{\delta}{2},
\end{align*}
and same unit variance
\begin{align*}
    \operatorname{Var}[s(X)]=\operatorname{Var}[s(Y)]=\frac{e^\top(\sigma^2 I_d)e}{\sigma^2}=\|e\|^2=1.
\end{align*}
Therefore,
\begin{equation*}
s(X)\sim\mathcal N\left(-\frac{\delta}{2},1\right), \qquad
s(Y) \sim \mathcal N\left( \frac{\delta}{2},1 \right).
\end{equation*}
Standardising the two laws of $s$ then gives
\begin{align*}
\mathbb Q(\mathcal H_j)
&=
\overline{\Phi}\left(\tau_j-\frac{\delta}{2}\right),
\qquad
\mathbb P(\mathcal H_j)= \overline{\Phi}\left(\tau_j+\frac{\delta}{2} \right),
\end{align*}
which, substituted into~\eqref{eq::grs_two_terms},
becomes
\begin{align*}
G_{j+1} &= \mathbb Q(\mathcal H_j)-\lambda_j\,\mathbb P(\mathcal H_j)\\
&=\overline{\Phi}\left(\tau_j-\frac{\delta}{2}\right) - \lambda_j \overline{\Phi}\left(\tau_j+\frac{\delta}{2} \right).
\end{align*}
Taking $j=K$ completes the proof.
\end{proof}

\begin{remark}[Recovering RMC]
In the special case of $K=1$ we recover RMC. In fact,
$\lambda_1=G_1=1$ yields $\tau_1=0$, and hence
\begin{align*}
\Pr(J\leq 1)
&= 1-G_2\\ &= 1- \overline{\Phi}\left( -\frac{\delta}{2} \right) + \overline{\Phi}\left( \frac{\delta}{2} \right)\\ &= 2\overline{\Phi}\left( \frac{\delta}{2} \right).
\end{align*}
This is the same
acceptance probability as \cite{de2025accelerated}.
\end{remark}

% \begin{figure}[t]
%   \centering
%   \includegraphics[width=0.65\linewidth]{Figures/curves/accept_prob.pdf}
%   \caption{Per-step acceptance probability of D-GRS for different number $K$ of draft proposals. }
%   \label{fig:accept_prob}
% \end{figure}
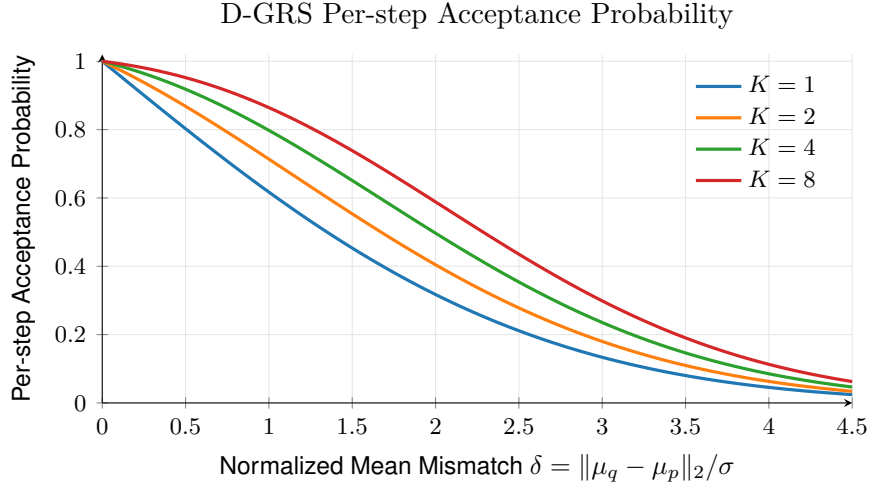
\begin{figure}[t]
\centering
%% ---------------------------------------------------------------------
%% Figure 6 -- per-step (node) acceptance probability of D-GRS.
%%
%% Curves are 1 - G_{K+1} from Theorem 2, evaluated with the recursion
%%   lambda_0 = 0,  G_1 = 1
%%   lambda_k = lambda_{k-1} + G_k
%%   tau_k    = ln(lambda_k) / delta
%%   G_{k+1}  = Phibar(tau_k - delta/2) - lambda_k * Phibar(tau_k + delta/2)
%% Sanity checks: K=1 reproduces 2*Phibar(delta/2) of eq. (20) to 6 d.p.,
%% and every curve tends to 1 as delta -> 0, per Remark 2.
%%
%% Requires in the preamble:
%%   \usepackage{pgfplots}
%%   \pgfplotsset{compat=1.18}
%% ---------------------------------------------------------------------
\begin{tikzpicture}
\begin{axis}[
  width=0.78\linewidth, height=0.42\linewidth,
  xlabel={Normalized Mean Mismatch $\delta=\|\mu_q-\mu_p\|_2/\sigma$},
  ylabel={Per-step Acceptance Probability},
  title={D-GRS Per-step Acceptance Probability},
  every axis label/.append style={font=\small\sffamily},
  tick label style={font=\footnotesize\sffamily},
  legend style={font=\footnotesize\sffamily, draw=none, fill=none,
                at={(0.97,0.97)}, anchor=north east},
  legend cell align=left,
  xmin=0, xmax=4.5, ymin=0, ymax=1.02,
  xtick={0,0.5,...,4.5}, ytick={0,0.2,...,1.0},
  grid=major, grid style={gray!20},
  axis lines=left, clip=false,
]

\addplot[grsone, very thick] coordinates {
   (0.00,1.00000) (0.05,0.98005) (0.10,0.96012) (0.15,0.94021) (0.20,0.92034)
   (0.25,0.90052) (0.30,0.88076) (0.35,0.86108) (0.40,0.84148) (0.45,0.82198)
   (0.50,0.80259) (0.55,0.78332) (0.60,0.76418) (0.65,0.74518) (0.70,0.72634)
   (0.75,0.70766) (0.80,0.68916) (0.85,0.67084) (0.90,0.65271) (0.95,0.63479)
   (1.00,0.61708) (1.05,0.59958) (1.10,0.58232) (1.15,0.56529) (1.20,0.54851)
   (1.25,0.53197) (1.30,0.51569) (1.35,0.49968) (1.40,0.48393) (1.45,0.46845)
   (1.50,0.45325) (1.55,0.43834) (1.60,0.42371) (1.65,0.40937) (1.70,0.39533)
   (1.75,0.38157) (1.80,0.36812) (1.85,0.35497) (1.90,0.34211) (1.95,0.32956)
   (2.00,0.31731) (2.05,0.30536) (2.10,0.29372) (2.15,0.28237) (2.20,0.27133)
   (2.25,0.26059) (2.30,0.25014) (2.35,0.23999) (2.40,0.23014) (2.45,0.22058)
   (2.50,0.21130) (2.55,0.20231) (2.60,0.19360) (2.65,0.18517) (2.70,0.17702)
   (2.75,0.16913) (2.80,0.16151) (2.85,0.15416) (2.90,0.14706) (2.95,0.14021)
   (3.00,0.13361) (3.05,0.12726) (3.10,0.12114) (3.15,0.11526) (3.20,0.10960)
   (3.25,0.10416) (3.30,0.09894) (3.35,0.09393) (3.40,0.08913) (3.45,0.08453)
   (3.50,0.08012) (3.55,0.07590) (3.60,0.07186) (3.65,0.06800) (3.70,0.06431)
   (3.75,0.06079) (3.80,0.05743) (3.85,0.05423) (3.90,0.05118) (3.95,0.04827)
   (4.00,0.04550) (4.05,0.04287) (4.10,0.04036) (4.15,0.03799) (4.20,0.03573)
   (4.25,0.03359) (4.30,0.03156) (4.35,0.02963) (4.40,0.02781) (4.45,0.02608)
   (4.50,0.02445)
};
\addlegendentry{$K=1$}

\addplot[grstwo, very thick] coordinates {
   (0.00,1.00000) (0.05,0.98828) (0.10,0.97620) (0.15,0.96377) (0.20,0.95101)
   (0.25,0.93794) (0.30,0.92455) (0.35,0.91087) (0.40,0.89691) (0.45,0.88268)
   (0.50,0.86821) (0.55,0.85351) (0.60,0.83859) (0.65,0.82347) (0.70,0.80818)
   (0.75,0.79272) (0.80,0.77711) (0.85,0.76138) (0.90,0.74554) (0.95,0.72961)
   (1.00,0.71360) (1.05,0.69755) (1.10,0.68145) (1.15,0.66534) (1.20,0.64923)
   (1.25,0.63313) (1.30,0.61707) (1.35,0.60105) (1.40,0.58511) (1.45,0.56925)
   (1.50,0.55349) (1.55,0.53784) (1.60,0.52232) (1.65,0.50694) (1.70,0.49172)
   (1.75,0.47667) (1.80,0.46180) (1.85,0.44713) (1.90,0.43266) (1.95,0.41840)
   (2.00,0.40437) (2.05,0.39057) (2.10,0.37702) (2.15,0.36371) (2.20,0.35066)
   (2.25,0.33788) (2.30,0.32536) (2.35,0.31311) (2.40,0.30115) (2.45,0.28946)
   (2.50,0.27806) (2.55,0.26695) (2.60,0.25613) (2.65,0.24559) (2.70,0.23535)
   (2.75,0.22539) (2.80,0.21573) (2.85,0.20636) (2.90,0.19727) (2.95,0.18848)
   (3.00,0.17996) (3.05,0.17173) (3.10,0.16377) (3.15,0.15609) (3.20,0.14868)
   (3.25,0.14154) (3.30,0.13465) (3.35,0.12803) (3.40,0.12166) (3.45,0.11554)
   (3.50,0.10966) (3.55,0.10402) (3.60,0.09860) (3.65,0.09342) (3.70,0.08845)
   (3.75,0.08370) (3.80,0.07916) (3.85,0.07482) (3.90,0.07067) (3.95,0.06672)
   (4.00,0.06295) (4.05,0.05936) (4.10,0.05593) (4.15,0.05268) (4.20,0.04959)
   (4.25,0.04665) (4.30,0.04385) (4.35,0.04120) (4.40,0.03869) (4.45,0.03631)
   (4.50,0.03406)
};
\addlegendentry{$K=2$}

\addplot[grsfour, very thick] coordinates {
   (0.00,1.00000) (0.05,0.99368) (0.10,0.98695) (0.15,0.97981) (0.20,0.97225)
   (0.25,0.96427) (0.30,0.95588) (0.35,0.94707) (0.40,0.93785) (0.45,0.92822)
   (0.50,0.91819) (0.55,0.90776) (0.60,0.89695) (0.65,0.88576) (0.70,0.87420)
   (0.75,0.86229) (0.80,0.85002) (0.85,0.83743) (0.90,0.82452) (0.95,0.81131)
   (1.00,0.79781) (1.05,0.78404) (1.10,0.77002) (1.15,0.75577) (1.20,0.74130)
   (1.25,0.72663) (1.30,0.71179) (1.35,0.69679) (1.40,0.68166) (1.45,0.66641)
   (1.50,0.65106) (1.55,0.63564) (1.60,0.62016) (1.65,0.60466) (1.70,0.58913)
   (1.75,0.57362) (1.80,0.55813) (1.85,0.54269) (1.90,0.52731) (1.95,0.51201)
   (2.00,0.49682) (2.05,0.48175) (2.10,0.46681) (2.15,0.45202) (2.20,0.43739)
   (2.25,0.42295) (2.30,0.40871) (2.35,0.39467) (2.40,0.38085) (2.45,0.36726)
   (2.50,0.35391) (2.55,0.34081) (2.60,0.32797) (2.65,0.31540) (2.70,0.30311)
   (2.75,0.29109) (2.80,0.27936) (2.85,0.26792) (2.90,0.25678) (2.95,0.24593)
   (3.00,0.23538) (3.05,0.22513) (3.10,0.21519) (3.15,0.20554) (3.20,0.19620)
   (3.25,0.18716) (3.30,0.17841) (3.35,0.16996) (3.40,0.16181) (3.45,0.15394)
   (3.50,0.14636) (3.55,0.13907) (3.60,0.13205) (3.65,0.12530) (3.70,0.11882)
   (3.75,0.11260) (3.80,0.10664) (3.85,0.10093) (3.90,0.09547) (3.95,0.09024)
   (4.00,0.08525) (4.05,0.08048) (4.10,0.07593) (4.15,0.07159) (4.20,0.06745)
   (4.25,0.06352) (4.30,0.05978) (4.35,0.05622) (4.40,0.05284) (4.45,0.04964)
   (4.50,0.04660)
};
\addlegendentry{$K=4$}

\addplot[grseight, very thick] coordinates {
   (0.00,1.00000) (0.05,0.99681) (0.10,0.99330) (0.15,0.98945) (0.20,0.98525)
   (0.25,0.98069) (0.30,0.97576) (0.35,0.97044) (0.40,0.96474) (0.45,0.95864)
   (0.50,0.95214) (0.55,0.94523) (0.60,0.93791) (0.65,0.93017) (0.70,0.92201)
   (0.75,0.91343) (0.80,0.90443) (0.85,0.89502) (0.90,0.88520) (0.95,0.87497)
   (1.00,0.86434) (1.05,0.85331) (1.10,0.84191) (1.15,0.83013) (1.20,0.81799)
   (1.25,0.80550) (1.30,0.79268) (1.35,0.77954) (1.40,0.76610) (1.45,0.75238)
   (1.50,0.73839) (1.55,0.72415) (1.60,0.70969) (1.65,0.69502) (1.70,0.68016)
   (1.75,0.66514) (1.80,0.64998) (1.85,0.63471) (1.90,0.61933) (1.95,0.60388)
   (2.00,0.58838) (2.05,0.57285) (2.10,0.55731) (2.15,0.54179) (2.20,0.52631)
   (2.25,0.51088) (2.30,0.49553) (2.35,0.48028) (2.40,0.46515) (2.45,0.45016)
   (2.50,0.43532) (2.55,0.42065) (2.60,0.40617) (2.65,0.39189) (2.70,0.37783)
   (2.75,0.36400) (2.80,0.35041) (2.85,0.33708) (2.90,0.32401) (2.95,0.31122)
   (3.00,0.29871) (3.05,0.28649) (3.10,0.27457) (3.15,0.26295) (3.20,0.25164)
   (3.25,0.24063) (3.30,0.22995) (3.35,0.21957) (3.40,0.20952) (3.45,0.19978)
   (3.50,0.19035) (3.55,0.18124) (3.60,0.17245) (3.65,0.16396) (3.70,0.15578)
   (3.75,0.14791) (3.80,0.14034) (3.85,0.13306) (3.90,0.12607) (3.95,0.11937)
   (4.00,0.11294) (4.05,0.10679) (4.10,0.10090) (4.15,0.09528) (4.20,0.08990)
   (4.25,0.08478) (4.30,0.07989) (4.35,0.07523) (4.40,0.07080) (4.45,0.06659)
   (4.50,0.06258)
};
\addlegendentry{$K=8$}
\end{axis}
\end{tikzpicture}
\caption{Per-step acceptance probability of D-GRS for different number $K$ of draft proposals.}
\label{fig:accept_prob}
\end{figure}

\begin{remark}[Degenerate and anisotropic cases]
\label{rem:degenerate}
Theorem~\ref{thm:acceptance_probability} assumes $\delta>0$, which is needed for
$\tau_j=\ln\lambda_j/\delta$ to be defined. If $\delta=0$ then
$\mathbb P=\mathbb Q$, $\rho\equiv1$, and the first proposal is accepted almost
surely, so $G_2=0$ and the acceptance probability equals $1$ for every
$K\ge1$; this is also the limit of
\eqref{eq::grs_acceptance_probability} as $\delta\downarrow0$. The argument
extends verbatim to any common covariance $\Sigma\succ0$ upon replacing $\delta$
by the Mahalanobis distance
$\delta=\|\Sigma^{-1/2}(\mu_q-\mu_p)\|$ and $s$ by
$s(y)=e^\top\Sigma^{-1/2}(y-\bar\mu)$ with
$e=\Sigma^{-1/2}(\mu_q-\mu_p)/\delta$.
\end{remark}

\section{Extended Experimental Results}
In this section, we provide more details about our experimental settings.

% \subsection{Impact of the Branching Factor $K$}
% \label{sec:branching-factor}
\subsection{Implementation Details}
\paragraph{The delayed-drift proposal.}
In our experiments, we instantiate the template with~\eqref{eq:delayed_drift}, taking the target drift evaluated at the root: $\tilde{m}^p_{n'}(y)=y+\gamma\, b^q_{t_n}(Y_n)$ for every step $n'\leq L$ of the
round. The proposal mean is therefore the same map at every depth, so the drafting phase in
Algorithm~\ref{alg:template} specializes to
\begin{equation*}
  \widehat Y_v \sim \mathcal N\big(\tilde{m}^p_{n}(\widehat Y_u),\,
  \sigma^2_{n+\ell-1}I_d\big),
  \  v\in\ch(u),\ |u|=\ell-1,
  \label{eq::frozen_drift}
\end{equation*}
with $\tilde{m}^p_{n}(\widehat Y_u)=\widehat Y_u + \gamma b^q_{t_n}(Y_n)$, $Y_n=Y_\rt$. 
Note that this would require 1 extra target call per round (other than the one needed for verification) to compute the target drift at the root node $b^q_{t_n}(Y_n)$. 

\paragraph{Root-drift Prefetching.}
When the delayed-drift is used as proposal, the drafting phase requires evaluating $b^q_{t_n}(Y_n)$ at the root of every round, that is, one target evaluation in addition to the one used for verification. Following \citet{de2025accelerated}, we avoid this by reusing a drift already computed during the preceding round, a device we refer to as \emph{root-drift prefetching}: $b^q_{t_n}(Y_n)$ is replaced by $b^q_{t_{n'}}(\tilde Y)$ for some pair $(n', \tilde Y)$ whose drift is already available.

\subsection{The Value of $K$ and $L$}
\label{sec:branching-factor}
\begin{figure}[t]
    \centering
    \includegraphics[width=0.8\linewidth]{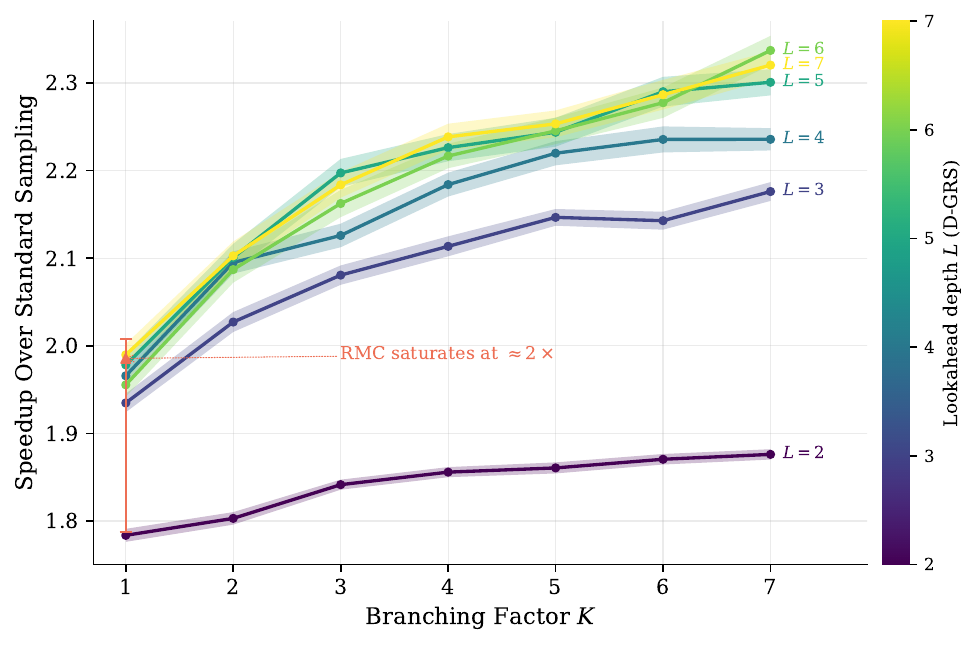}
    \caption{NFE speedup over standard sampling on the Gaussian-mixture target as a function of tree width $K$. Points indicate the mean per-replicate speedup, averaged over 100 independent trials. Coloured lines depict D-GRS across lookahead depths $L\in\{2,\dots,7\}$ (see color bar), with shaded bands showing $\pm 1$ standard error.}

    \label{fig:gm_speedup_vs_k}
\end{figure}

We next isolate the two mechanisms by which a proposal budget may be allocated:
the branching factor $K$, i.e.\ the number of candidate children expanded per
node, and the lookahead depth $L$. Figure~\ref{fig:gm_speedup_vs_k} reports the
speedup over standard sampling as a function of $K$, stratified by $L$, with
each point averaged over $100$ independent replicates.

RMC admits no branching, as it drafts a single linear chain, and consequently
saturates at $\approx 2\times$ irrespective of the chain length. This ceiling admits
a closed-form characterization. Under the assumption that the per-step acceptance probability of the coupling $\alpha$ is fixed for each step, the expected length of the accepted prefix of a
chain of $n$ drafted states is
\begin{equation*}
    \sum_{i=1}^{n} \alpha^{i}
    \;=\; \frac{\alpha\left(1-\alpha^{n}\right)}{1-\alpha}
    \;\xrightarrow[n\to\infty]{}\; \frac{\alpha}{1-\alpha},
\end{equation*}
so that, since each round commits the accepted prefix together with one
resampled state at the cost of a single target evaluation, the number of
committed steps per target network function evaluation is bounded above by
$1/(1-\alpha)$. The measured acceptance probability $\alpha \approx 0.52$ yields
a predicted ceiling of $2.09\times$, in close agreement with the observed
$2\times$. Extending the chain beyond this regime therefore incurs additional
draft computation without any attendant reduction in target evaluations: the
geometric decay of the acceptance probability along the chain, rather than the
draft budget, constitutes the binding constraint.

%Branching relaxes precisely this constraint. 
The relative value of depth $L$ and width $K$ is governed by the per-step acceptance probability of the coupling, which is controlled by the ratio $\delta$ defined in \eqref{eq:r1proj-variables}. 
% The dimensionality $d$, the step size $\gamma$ and the
% stochasticity level $\varepsilon$ all widen this ratio, lowering the per-step
% acceptance probability and thereby increasing the marginal value of each
% additional child.
Consequently, $K,L$ are complementary rather than interchangeable,
and their relative priority is separated by a problem-dependent crossover. Below
it, depth is the binding constraint: a shallow tree exhausts its lookahead before
branching can be exploited, and no $K$ recovers the deficit. At
$L=2$ the best configuration attains only $1.876\times$, short of the RMC ceiling,
whereas at $K=1$ the tree degenerates to a chain and D-GRS reproduces that
ceiling. D-GRS surpasses RMC only once the lookahead is deep enough to sustain an accepted run, here from $(K,L)=(2,3)$ onward. Above the crossover, the two parameters compete for the same budget: the marginal benefit of depth is exhausted by $L\approx5$, beyond which the $L\in\{5,6,7\}$ curves lie within one standard error of one another, while widening from $K=5$ to $K=7$ at $L=6$ still yields $0.09\times$, up to a best of $2.34\times$ at $(K,L)=(7,6)$. Once the lookahead is sufficient, additional budget is better spent on width than on depth.

\subsection{SD3 Conversion}\label{sec:sd3_conversion}
Unlike DDPMs, SD3 is trained as a rectified flow (flow matching)
model, i.e. it directly parameterizes a probability-flow ODE. We recall here how this ODE is recast as the SDE in \eqref{eq:sde}, which requires identifying the drift $f_t$, the diffusion
coefficient $g_t$ and the score $s_t$ implied by the SD3 interpolant.

Unlike DDPMs, SD3 is trained as a rectified flow model: it parameterises a
probability-flow ODE, whose transitions are deterministic point masses, yielding a total variation distance of one. Speculation is therefore vacuous for the native
SD3 sampler, and applying our framework requires transitions that are genuinely
Gaussian.

We obtain them without retraining, by exhibiting SD3 as a member of the family of
reverse-time SDEs~\eqref{eq:sde} that share its marginals. Concretely, the
conversion identifies the drift $f_t$, the diffusion coefficient $g_t$ and the
score $s_t$ implied by the SD3 interpolant, so that the reverse drift~\eqref{eq:sde} and its
Euler--Maruyama discretisation~\eqref{eq:em_discretisation}--\eqref{eq::target_mean} apply verbatim. The outcome, recorded
in~\eqref{eq:sd3-mean}--\eqref{eq:sd3-scale} below, is a Gaussian transition
kernel whose variance is independent of the state, which is precisely the
structure the rank-$1$ reduction~\eqref{eq:r1proj-variables}--\eqref{eq:final_proj} requires.

\paragraph{Interpolant and velocity field.}
Let $X_0\sim p_{\rm data}$ and $\xi\sim\mathcal N(0,\mathbb I_d)$ be independent.
For $t\in[0,1]$, SD3 perturbs data along the linear interpolant
\begin{align}
  X_t=\alpha_t X_0+\beta_t\,\xi,
  \quad
  \alpha_t=1-t,
  \quad
  \beta_t=t,
  \label{eq:sd3-interpolant}
\end{align}
so that $X_0\sim p_{\rm data}$ and $X_1\sim\mathcal N(0,\mathbb I_d)$. Along
\eqref{eq:sd3-interpolant} the velocity is constant,
$\tfrac{\diff X_t}{\diff t}=\xi-X_0$, and the network $v_\theta$ is trained by
conditional flow matching to regress it from the corrupted state alone. Since the
pair $(X_0,\xi)$ is not identifiable from $X_t$, the minimiser of the training
objective is the conditional expectation, that is the \emph{marginal} velocity field
\begin{align}
  v_t(x)
  &:=\mathbb E\!\left[\frac{\diff X_t}{\diff t}\,\bigg|\,X_t=x\right]\\
  &=\mathbb E[\xi\mid X_t=x]-\mathbb E[X_0\mid X_t=x].
  \label{eq:marginal-velocity}
\end{align}
% The field $v_t$ generates the marginals $p_t:=\mathrm{Law}(X_t)$ through the probability-flow ODE
% $\diff X_t=v_t(X_t)\,\diff t$, which is SD3's native sampler.
 
\paragraph{Drift and diffusion coefficient.}
Any affine interpolant of the form~\eqref{eq:sd3-interpolant} is the marginal law
of a linear SDE 
$$\diff X_t=f_tX_t\,\diff t+g_t\,\diff B_t,$$ since for such an SDE
the conditional mean and variance of $X_t\mid X_0$ obey
$\tfrac{\diff\alpha_t}{\diff t}=f_t\alpha_t$ and
$\tfrac{\diff\beta_t^2}{\diff t}=2f_t\beta_t^2+g_t^2$. Inverting these relations,
\begin{align}
  f_t=\frac{1}{\alpha_t}\frac{\diff\alpha_t}{\diff t},
  \qquad
  g_t^2=2\beta_t\!\left(\frac{\diff\beta_t}{\diff t}-f_t\beta_t\right),
  \label{eq:fg-from-interpolant}
\end{align}
and substituting $\alpha_t=1-t$, $\beta_t=t$ yields the SD3 coefficients
\begin{align}
  f_t(x)=-\frac{x}{1-t},
  \qquad
  g_t^2=\frac{2t}{1-t}.
  \label{eq:sd3-fg}
\end{align}
 
\paragraph{Score.}
The reverse drift~\eqref{eq:reverse_drift} is expressed through the Stein score, which the flow
matching parameterization does not supply directly. It is, however, an affine
function of the velocity. Taking
conditional expectations in~\eqref{eq:sd3-interpolant} gives
$$x=(1-t)\,\mathbb E[X_0\mid X_t=x]+t\,\mathbb E[\xi\mid X_t=x].$$
Eliminating $\mathbb E[X_0\mid X_t=x]$ from~\eqref{eq:marginal-velocity} then
leaves
\begin{align}
  v_t(x)=\frac{\mathbb E[\xi\mid X_t=x]-x}{1-t}.
  \label{eq:v-from-xi}
\end{align}
Combining \eqref{eq:v-from-xi} with $s_t(x)=-\mathbb E[\xi\mid X_t=x]/\beta_t$ \cite{karras2022elucidating}, we have
\begin{align}
  s_t(x)=-\frac{x+(1-t)\,v_t(x)}{t}.
  \label{eq:score-from-velocity}
\end{align}
% A single evaluation of $v_\theta$ therefore supplies both the velocity and the
% score.
 
\paragraph{Reverse process.}
Substituting $f_{1-t}(x)=-x/t$, $g^2_{1-t}=2(1-t)/t$
and~\eqref{eq:score-from-velocity}, all evaluated at time $1-t$, into the reverse
drift~\eqref{eq:reverse_drift} gives
\begin{align}
  b^q_t(x)=-\big(1+\varepsilon^2\big)v_{1-t}(x)-\varepsilon^2\,\frac{x}{t},
  \label{eq:sd3-reverse-drift}
\end{align}
and the reverse process~\eqref{eq:sde} is
$$\diff Y_t=b^q_t(Y_t)\,\diff t+\varepsilon\,g_{1-t}\,\diff W_t,$$
with $g_{1-t}=\sqrt{2(1-t)/t}$.

%so that the reverse process~\eqref{eq:sde} reads $\diff Y_t=b^q_t(Y_t)\,\diff t+\varepsilon\,g_{1-t}\,\diff W_t$ with $g_{1-t}=\sqrt{2(1-t)/t}$. Two limits confirm the identification: 
% At $\varepsilon=0$ expression~\eqref{eq:sd3-reverse-drift} reduces to $-v_{1-t}(x)$,
% the probability-flow ODE integrated backwards and hence SD3's native sampler,
% while at $\varepsilon=1$ it gives $-2v_{1-t}(x)-x/t$, the standard reverse-time
% SDE. 
% Every member of the family leaves the marginals $(p_t)_{t\in[0,1]}$
% invariant, so the churn level may be chosen freely and no retraining is required.
 
\paragraph{Transition kernels.}
Applying the Euler--Maruyama discretisation~(4) on the uniform grid
$t_n=n\gamma$ yields, conditionally on $Y_n=y$, the Gaussian transition
$\mathbb Q_n(\cdot\mid y)=\mathcal N\big(m^q_n(y),\sigma_n^2\mathbb I_d\big)$
of~\eqref{eq::target_mean}, with
\begin{align}
  m^q_n(y)
  &=y-\gamma\left[\big(1+\varepsilon^2\big)v_{1-t_n}(y)
  +\varepsilon^2\,\frac{y}{t_n}\right],
  \label{eq:sd3-mean}\\
  \sigma_n
  &=\sqrt{\gamma}\,\varepsilon\,g_{1-t_n}
  =\varepsilon\sqrt{\frac{2\gamma\,(1-t_n)}{t_n}}.
  \label{eq:sd3-scale}
\end{align}
The scale~\eqref{eq:sd3-scale} is a function of the time grid and the churn level
alone, and in particular does not depend on the state. A draft kernel obtained by
substituting a delayed velocity into~\eqref{eq:sd3-mean} consequently shares the
variance of the target kernel and differs from it only in its mean, which is the
structure required by~\eqref{eq:r1proj-variables}--\eqref{eq:final_proj}. Both RMC and D-GRS therefore apply to SD3 without modification.
 
\begin{remark}[The role of churn]
\label{rem:churn}
The churn level cannot be taken to zero. At $\varepsilon=0$ the
scale~\eqref{eq:sd3-scale} vanishes, both transitions degenerate to point masses,
and their total variation distance equals one, so every proposal is rejected and
speculation reduces to the standard sampler.

Its effect on acceptance is not
monotone: raising $\varepsilon$ increases the denominator of
$\delta=\|m^q_n-m^p_n\|/\sigma_n$, but it also increases the numerator, both through the factor $1+\varepsilon^2$ in~\eqref{eq:sd3-mean} and through the larger displacement of the drafted states from the point at which the delayed velocity was evaluated.
%The second effect dominates empirically, and acceptance degrades as $\varepsilon$ grows across the sweep $\varepsilon\in\{0.1,0.3,0.6\}$ reported in Section~5.2.
\end{remark}
 
\paragraph{Timestep shift.}
Since SD3 is trained at $1024\times1024$ resolution and its noise schedule is
resolution dependent. Following~\citet{esser2024scaling}, we adapt it to the $512\times512$ generation of our experimental setting by applying the shift
\begin{align}
  t\ \longmapsto\ \frac{\kappa\,t}{1+(\kappa-1)\,t},
  \qquad \kappa=3.0,
  \label{eq:timestep-shift}
\end{align}
to the uniform grid before evaluating $v_\theta$. The conversion above is
unaffected, since~\eqref{eq:score-from-velocity}
and~\eqref{eq:sd3-reverse-drift} hold pointwise in $t$.

\subsection{Qualitative Assessment of Correctness}
In order to qualitatively asses the correctness of the proposed algorithm D-GRS and the baseline implementation RMC \citep{de2025accelerated} in the context of image generation, we show 64 of the generated images for the experimental setting reported in Section~\ref{sec:experiments}. 
Figures~\ref{fig:cifar10-images} and ~\ref{fig:ffhq-images}
%~\ref{fig:sd3-images} 
show the generated images for CIFAR-10, FFHQ, 
%and COCO 2014, 
respectively. Table~\ref{tab:prompts} reports the prompts from COCO 2014 used for the conditional image generation using SD3, and Figure~\ref{fig:sd3-images} shows the generated images.

\newlength{\gridw}\setlength{\gridw}{0.30\linewidth}
\newcommand{\rowlabel}[1]{%
  \raisebox{\dimexpr 0.5\gridw-0.5\height\relax}{\rotatebox{90}{\small #1}}}

\newcommand{\vcenr}[1]{\raisebox{3.3\height}{#1}}
\newcommand{\vcend}[1]{\raisebox{2.3\height}{#1}}

\begin{figure}[t]
  \centering
  \setlength{\tabcolsep}{2pt}
  \renewcommand{\arraystretch}{1.0}
  \begin{tabular}{@{}c@{\hspace{4pt}}ccc@{}}
    & $\varepsilon = 0.1$ & $\varepsilon = 0.3$ & $\varepsilon = 0.6$ \\[2pt]
    \vcenr{\rotatebox{90}{\small RMC}} &
      \includegraphics[width=0.30\linewidth]{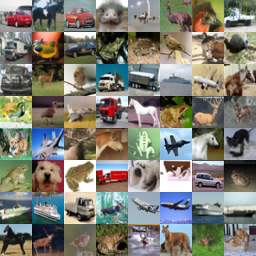} &
      \includegraphics[width=0.30\linewidth]{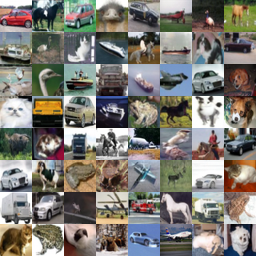} &
      \includegraphics[width=0.30\linewidth]{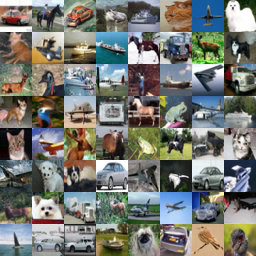} \\[2pt]
    \vcend{\rotatebox{90}{\small D-GRS}} &
      \includegraphics[width=0.30\linewidth]{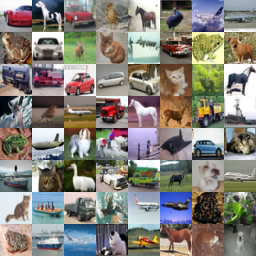} &
      \includegraphics[width=0.30\linewidth]{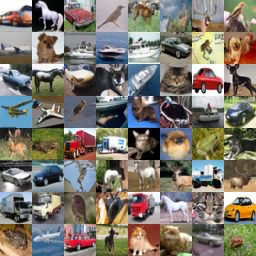} &
      \includegraphics[width=0.30\linewidth]{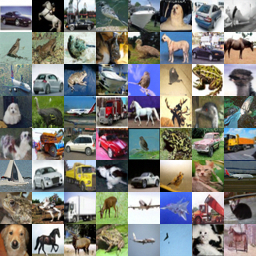} \\
  \end{tabular}
  \caption{64 CIFAR-10 samples (32$\times$32) selected at random generated at $K=4$, $L=4$, $T=100$, for increasing churn $\varepsilon$. Top row: RMC. Bottom row: D-GRS.}
  \label{fig:cifar10-images}
\end{figure}

\begin{figure}[t]
  \centering
  \setlength{\tabcolsep}{2pt}
  \renewcommand{\arraystretch}{1.0}
  \begin{tabular}{@{}c@{\hspace{4pt}}ccc@{}}
    & $\varepsilon = 0.1$ & $\varepsilon = 0.3$ & $\varepsilon = 0.6$ \\[2pt]
    \vcenr{\rotatebox{90}{\small RMC}} &
      \includegraphics[width=0.30\linewidth]{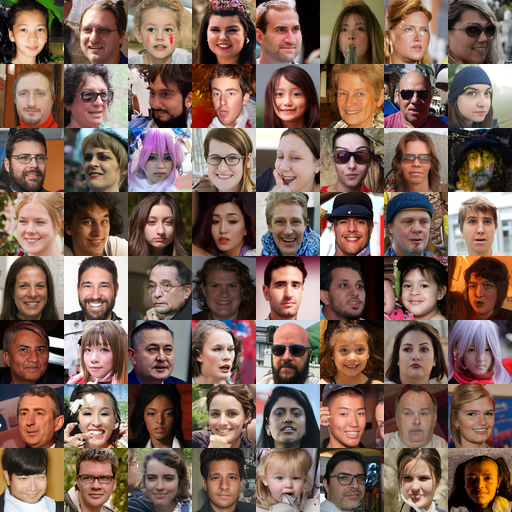} &
      \includegraphics[width=0.30\linewidth]{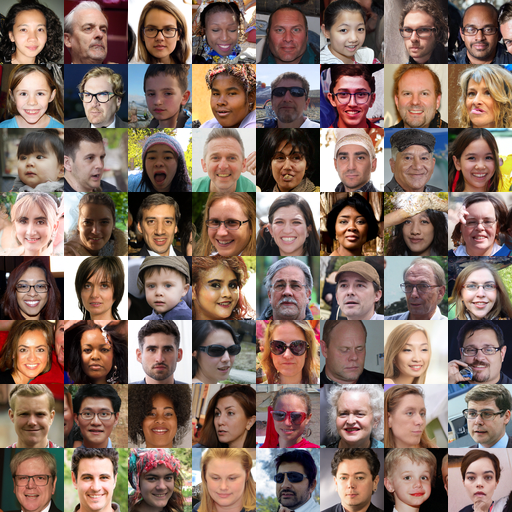} &
      \includegraphics[width=0.30\linewidth]{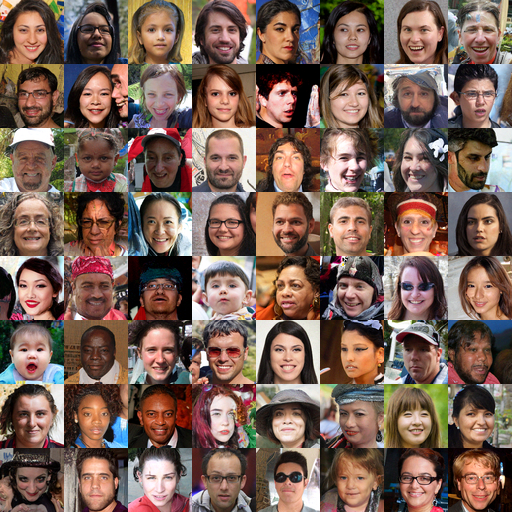} \\[2pt]
    \vcend{\rotatebox{90}{\small D-GRS}} &
      \includegraphics[width=0.30\linewidth]{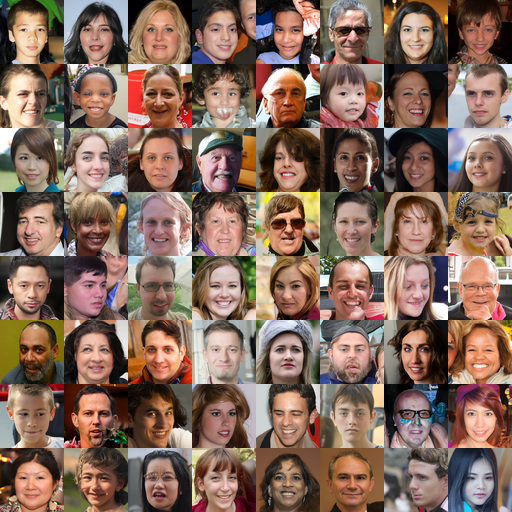} &
      \includegraphics[width=0.30\linewidth]{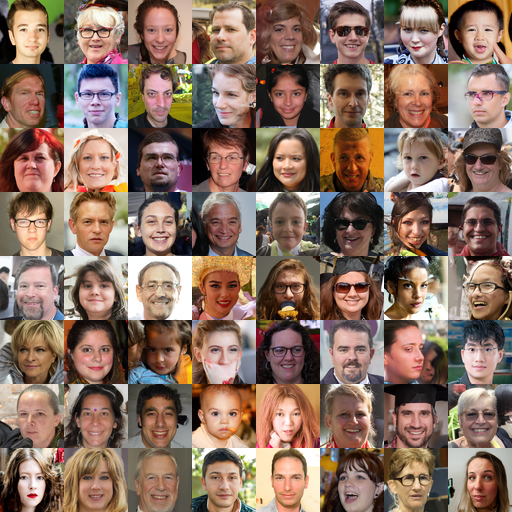} &
      \includegraphics[width=0.30\linewidth]{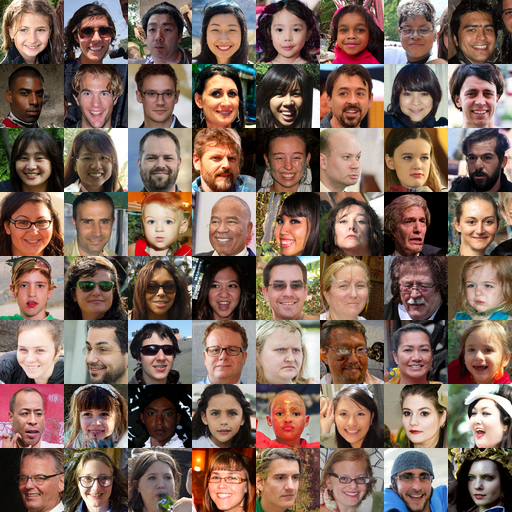} \\
  \end{tabular}
  \caption{First 64 FFHQ samples (64$\times$64) generated samples at $K=4$, $L=4$, $T=100$, for increasing churn $\varepsilon$. Top row: RMC. Bottom row: D-GRS.}
  \label{fig:ffhq-images}
\end{figure}

\begin{figure}
    \centering
    \begin{subfigure}{0.48\linewidth}
        \includegraphics[width=\linewidth]{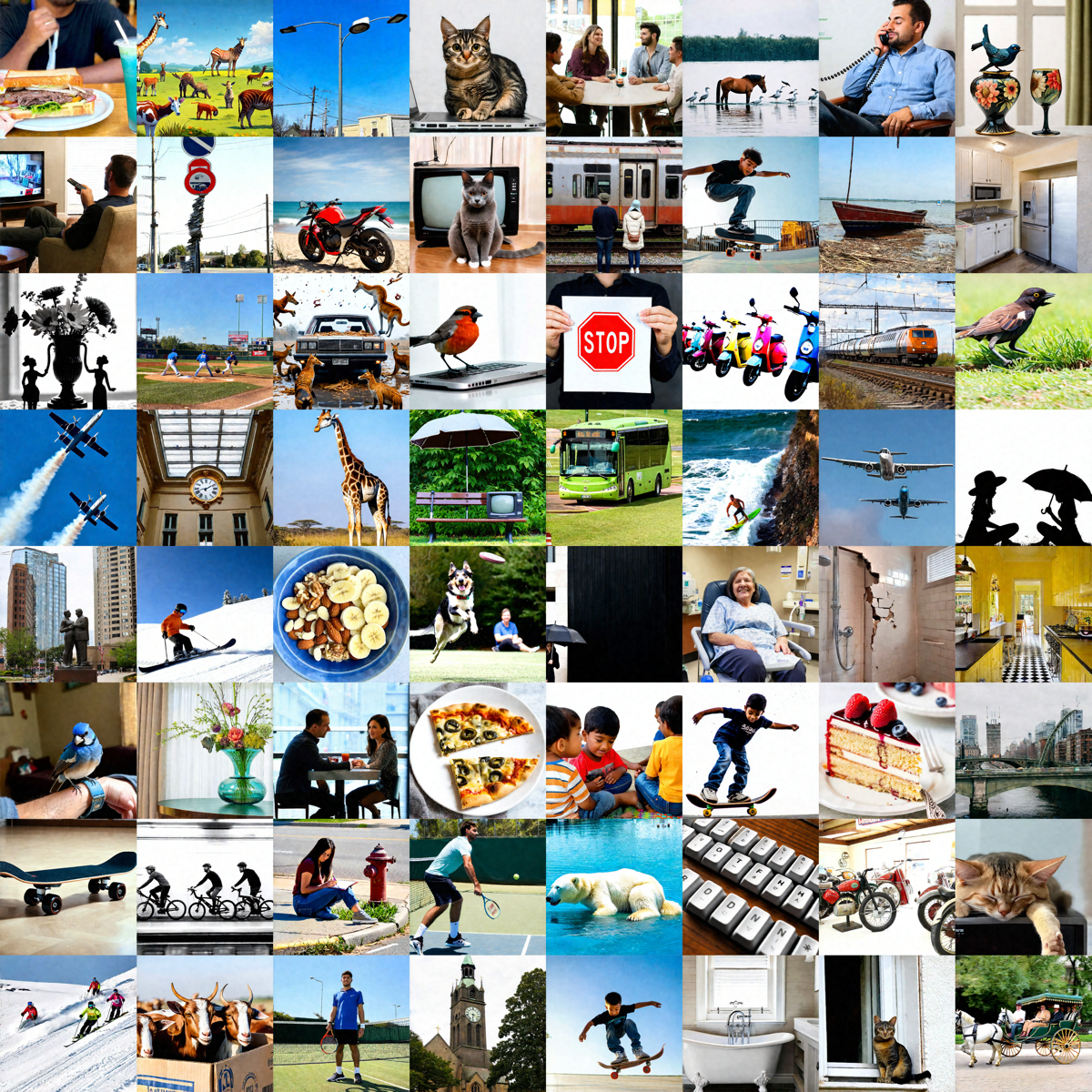}
        \caption{RMC}
        \label{fig:sub-a}
    \end{subfigure}
    \hfill
    \begin{subfigure}{0.48\linewidth}
        \includegraphics[width=\linewidth]{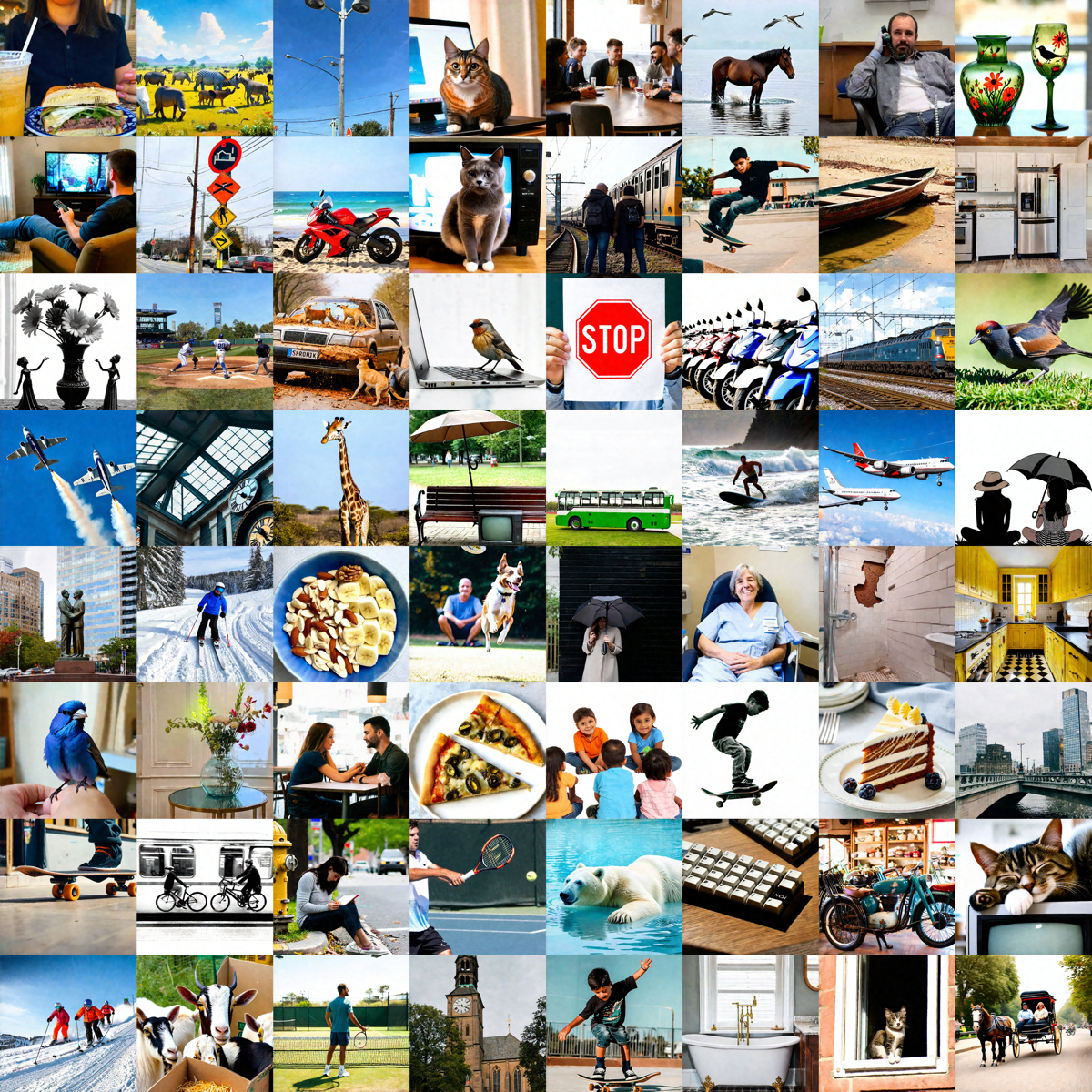}
        \caption{D-GRS}
        \label{fig:sub-b}
    \end{subfigure}
    \caption{First 64 images 512$\times$512 selected at random generated by SD3 conditionally on the prompts reported in Table~\ref{tab:prompts}, using (a) RMC and (b) D-GRS with $(K=1, L=56)$ and $(K=7, L=2)$, respectively, in the setting described in Section~\ref{sec:exp-sd3}, i.e., $T=50$, free-guidance of $7$, churn $\varepsilon=0.8$.}
    \label{fig:sd3-images}
\end{figure}
\newcolumntype{P}[1]{>{\raggedright\arraybackslash}p{#1}}

\begin{table}[p]
    \centering
    \scriptsize
    \setlength{\tabcolsep}{3pt}
    \renewcommand{\arraystretch}{0.9}
    \begin{tabular}{@{}r P{0.41\linewidth} r P{0.41\linewidth}@{}}
        \toprule
         & Prompts &  & \\
        \midrule
        1  & a person at a table with a sandwich and a drink & 33 & A Ma and Pa statue outside some high rise office buildings \\
        2  & A land with a lot of different animals. & 34 & A guy is skiing there is a lot of snow. \\
        3  & A street light at an intersection in a small town. & 35 & a large blue bowl of almonds, bananas and walnuts \\
        4  & A cat sitting on top of a laptop computer. & 36 & Dog jumping in air to catch flying disc with adult sitting in background. \\
        5  & A group of people sitting around a table together & 37 & A woman holding an umbrella standing next to a black wall. \\
        6  & A horse in the water next to the birds & 38 & A lady is in the hospital and sitting in a chair smiling. \\
        7  & a man sitting down with a telephone up to his ear & 39 & A bathroom that has a broken wall in the shower. \\
        8  & A vase with a bird on it and a matching wineglass with flowers on it. & 40 & A long, narrow yellow kitchen with black and white floor tiles. \\
        9  & A man sitting in a chair holding a remote and watching television in a living room. & 41 & a blue bird sitting on someone's arm in a room \\
        10 & a stack of traffic signs on a pole next to a street. & 42 & A table with a huge glass vase and fake flowers come out of it. \\
        11 & A red motorcycle parked next to a beach near the ocean. & 43 & A man sitting next to a woman at a table. \\
        12 & The grey cat with white feet sits on a televsion set. & 44 & A few sliced of olive pizza sitting on a white plate. \\
        13 & A couple of people that are standing near a train. & 45 & A group of children sitting around each other. \\
        14 & A young man doing a skateboard trick outside. & 46 & an image of boy on a skateboard doing tricks \\
        15 & A boat sitting on the ground, where their use to be water. & 47 & A slice of layer cake on a plate that is garnished \\
        16 & A kitchen area with two refrigerators and a microwave. & 48 & A bridge in a city on an overcast day. \\
        17 & Black and White picture of figurines around a vase with flowers & 49 & A skateboard that has its wheels on the floor. \\
        18 & some baseball players are playing baseball on a field & 50 & A group of bicycles on a subway train. \\
        19 & Animals surrounding a car and making a large mess. & 51 & The woman sits on the curb writing on a notepad near a fire hydrant. \\
        20 & a bird stands on top of a laptop keyboard & 52 & a man getting ready to hit a tennis ball with a racket \\
        21 & A person holding up a paper with a red stop sign on it. & 53 & A white polar bear laying on top of a pool of water. \\
        22 & This is an image of a row of scooters & 54 & There are keyboard keys on a wooden table. \\
        23 & a train on a train track with a sky background & 55 & A variety of old motorcycles on display in a shop \\
        24 & A bird aggressively protects his patch of turf. & 56 & A cat sleeping on the tv with his paw hanging down. \\
        25 & Two airplanes flying through a blue sky with smoke pouring out of their rear ends. & 57 & A group of people riding skis down a snow covered slope. \\
        26 & a building with a clock near the ceiling and skylights & 58 & A bunch of goats are eating out of a box \\
        27 & a very tall giraffe standing in the wilderness & 59 & A man standing on a grass tennis court and holding a tennis racket in his hand. \\
        28 & An umbrella lies behind a park bench that a small television is on. & 60 & A large clock tower over a church next to trees. \\
        29 & a lime and green bus parked on a grass field parking lot. & 61 & A young boy doing a trick on his skate board. \\
        30 & A man is surfboarding right up to the steep shoreline. & 62 & A white bath tub sitting under a window next to a sink. \\
        31 & one plane flies upside down over another plane & 63 & A cat is sitting outside on the ledge of a window. \\
        32 & Two girls sitting on the ground, one with a hat on, the other with an open umbrella. & 64 & Some people are riding along in a horse drawn Carriage \\
        \bottomrule
    \end{tabular}
    \caption{The first 64 prompts from COCO 2014~\cite{coco} used to conditionally generate the images shown in Figure~\ref{fig:sd3-images}, following the experimental setting of Section~\ref{sec:exp-sd3}.}    \label{tab:prompts}
\end{table}

\end{document}